%% file: eccv2022submission.tex
\documentclass[runningheads]{llncs}
\usepackage{graphicx}

\usepackage{tikz}
\usepackage{comment}
\usepackage{color}
\usepackage{algorithm}
\usepackage{algorithmicx}
\usepackage{algpseudocode}
\usepackage{booktabs}
\usepackage{multicol}
\usepackage{multirow}
\usepackage{xfrac}
\usepackage{breakcites}
\usepackage{subcaption}
\input{math_commands}

\usepackage{amssymb,amsthm,mathtools}

\usepackage[accsupp]{axessibility}  
\usepackage[pagebackref,breaklinks,colorlinks]{hyperref}

\newcommand{\STAB}[1]{\begin{tabular}{@{}c@{}}#1\end{tabular}}

\makeatletter
\def\@fnsymbol#1{\ensuremath{\ifcase#1\or \dagger\or *\or \ddagger\or
   \mathsection\or \mathparagraph\or \|\or **\or \dagger\dagger
   \or \ddagger\ddagger \else\@ctrerr\fi}}

\begin{document}
\pagestyle{headings}
\mainmatter
\def\ECCVSubNumber{8140}  

\title{Learning with Noisy Labels by \\Efficient Transition Matrix Estimation to Combat Label Miscorrection} 



\titlerunning{Learning with Noisy Labels by Efficient Transition Matrix Estimation}
%
\author{Seong Min Kye\thanks{Equal contribution.}\index{Kye, Seong Min} \and
Kwanghee Choi\textsuperscript{$\dagger$} \and
Joonyoung Yi\textsuperscript{$\dagger$} \and
Buru Chang\thanks{Corresponding author.}}
\authorrunning{S. Kye et al.}
%
\institute{Hyperconnect Inc.\\
\email{\{harris,kwanghee.choi,joonyoung.yi,buru.chang\}@hpcnt.com}}
\maketitle

\input{Sections/0_Abstract}
\input{Sections/1_Introduction}
\input{Sections/2_Related_Work}
\input{Sections/3_Method}
\input{Sections/4_Experiments}
\input{Sections/5_Conclusion}

\clearpage
%
%
\bibliographystyle{splncs04}
\bibliography{egbib,custom}

\newpage
\clearpage
\appendix
\input{Sections/A_Theoretical_Analysis}

\input{Sections/B_Experimental_Details}
\input{Sections/C_Additional_Experiments}

\input{Sections/D_Additional_Related_Work}

\end{document}

%% file: math_commands.tex
\usepackage{amsmath,amsfonts,bm}

\def\eqref#1{equation~\ref{#1}}

\def\1{\bm{1}}

\newcommand{\test}{\mathcal{D_{\mathrm{test}}}}

\def\mT{{\bm{T}}}

\DeclareMathAlphabet{\mathsfit}{\encodingdefault}{\sfdefault}{m}{sl}
\SetMathAlphabet{\mathsfit}{bold}{\encodingdefault}{\sfdefault}{bx}{n}

\def\emT{{T}}



%% file: Sections/0_Abstract.tex
\begin{abstract}\label{sec:0_abstract}
Recent studies on learning with noisy labels have shown remarkable performance by exploiting a small clean dataset.
In particular, model agnostic meta-learning-based label correction methods further improve performance by correcting noisy labels on the fly.
However, there is no safeguard on the label miscorrection, resulting in unavoidable performance degradation.
Moreover, every training step requires at least three back-propagations, significantly slowing down the training speed.
To mitigate these issues, we propose a robust and efficient method, \textit{FasTEN}, which learns a label transition matrix on the fly.
Employing the transition matrix makes the classifier skeptical about all the corrected samples, which alleviates the miscorrection issue.
We also introduce a two-head architecture to efficiently estimate the label transition matrix every iteration within a single back-propagation, so that the estimated matrix closely follows the shifting noise distribution induced by label correction.
Extensive experiments demonstrate that our FasTEN shows the best performance in training efficiency while having comparable or better accuracy than existing methods, especially achieving state-of-the-art performance in a real-world noisy dataset, Clothing1M.

\keywords{Learning with noisy labels; Label correction; Transition matrix estimation}
\end{abstract}

%% file: Sections/1_Introduction.tex
\section{Introduction}\label{sec:1_introduction}
In the last decade, supervised learning has achieved great success by leveraging an abundant amount of annotated data to solve various classification tasks such as image classification~\cite{he2016deep}, object detection~\cite{girshick2014rich}, and face recognition~\cite{taigman2014deepface}.
It has been proven both theoretically and empirically that the performance of supervised learning-based classification models steadily improves as the size of annotated data increases~\cite{goodfellow2016deep,charikar2017learning,floridi2020gpt}.
However, we cannot avoid \textit{noisy labels} due to its coarse-grained annotation sources~\cite{hendrycks2018using_glc,zheng2021meta}, resulting in performance degradation~\cite{chen2019understanding}.

\input{Figures/1_Efficiency}
Many methods have been proposed to build a classifier that is robust to noisy labels.
Unlike traditional methods~\cite{mnih2012learning,van2015learning,azadi2015auxiliary,patrini2016loss} which assume that all the given labels are potentially corrupted, recently proposed methods utilize an inexpensively obtained small clean dataset to improve performance further.
Based on the clean data set, loss correction methods~\cite{hendrycks2018using_glc,wang2020training} reduce the influence of noisy labels by modifying loss functions and re-weighting methods~\cite{ren2018learning,shu2019meta,bahri2020deep,ghosh2021we} penalize samples that are likely to be noisy labels.
Especially, recent label correction methods~\cite{wu2020learning,zheng2021meta} achieve remarkable performance based on model-agnostic meta-learning (MAML)~\cite{finn2017model}.
These methods relabel noisy labels to directly reduce the noise level, raising the theoretical upper bound of the predictive performance (See Appendix~\ref{appendix:theory_motivation}).

However, there are two challenges for these MAML-based label correction methods:
(1) \textit{The label correction methods blindly trust the already miscorrected labels}.
Erroneously corrected labels are often kept throughout the training, which causes the model to learn the miscorrected labels as ground-truth labels.
Several studies~\cite{mirzasoleiman2020coresets,wu2020learning} attempt to tackle this through training techniques such as soft labels, whereas it does not fundamentally solve the problem.
(2) \textit{MAML-based methods are inherently slow in training, resulting in excessive computational overhead}.
The inefficiency comes from multiple training steps per single iteration of MAML-based methods, including virtual updates with inner optimization loops.

To alleviate these issues, we propose a robust and efficient method called \textbf{\textit{FasTEN}} (\textbf{Fas}t \textbf{T}ransition Matrix \textbf{E}stimation for Learning with \textbf{N}oisy Labels).
FasTEN efficiently estimates a transition matrix to learn with noisy labels while continuously correcting them on-the-fly.
It is theoretically proven that the correctly estimated label transition matrix is useful to obtain a statistically consistent classifier from noisy labels~\cite{xia2019anchor,yao2020dual} (See Appendix~\ref{appendix:theory_background}), i.e., more robust to noisy labels.
To efficiently estimate the transition matrix, we adopt a two-head architecture that consists of two classifiers, a noisy and a clean classifier, with a shared feature extractor.
For every iteration, the noisy classifier estimates the label transition matrix shifted by the label correction.
On the other hand, the clean classifier is trained to be statistically consistent by leveraging the estimated transition matrix.
Using the output of the clean classifier, FasTEN relabels noisy labels to reduce the noise level.
Our proposed FasTEN has a safeguard for the miscorrected labels since it adaptively estimates the transition matrix on every iteration, so that the clean classifier stays equally skeptical towards all the corrected labels.
Furthermore, our efficient method jointly optimizes the two-head architecture with only a single back-propagation for each iteration, boosting training speed.
In this paper, we focus on solving the problem of \textit{class-dependent} noisy labels~\cite{goldberger2016training,patrini2017making,zhang2021learning} (i.e., $p ( \Bar{y} | y, x ) = p ( \Bar{y} | y)$), although the problem of instance-dependent noisy labels~\cite{xia2020part,cheng2021learning,zhu2021second} remains an important problem to be addressed.

Experimental results show that our method achieves state-of-the-art performance by a large margin on both the synthetic and real-world noisy label datasets, various noise levels of \textit{CIFAR}~\cite{krizhevsky2009learning} and \textit{Clothing1M}~\cite{xiao2015learning}, respectively.
We demonstrate the exceptional training speed of our proposed FasTEN while achieving better performance compared to baselines, as shown in Figure~\ref{fig:1_efficiency}.
Especially, although our FasTEN assumes only class-dependent noisy labels, it also achieves state-of-the-art performance in the Clothing 1M dataset which contains instance-dependent noisy labels.
This experimental result supports recent observations that leveraging the accurately estimated transition matrix with small clean data is helpful for alleviating instance-dependent noise~\cite{menon2016learning,hendrycks2018using_glc,zhu2021second,jiang2022an} (See Appendix~\ref{appendix:theory_background}).
Finally, we conduct a thorough analysis to understand the inner mechanisms of our proposed method.

Our contribution in this paper is threefold:
(1) We propose a robust and efficient method that learns a transition matrix to learn with noisy labels while continuously correcting them on the fly.
To the best of our knowledge, this is the first attempt to improve the label correction with the transition matrix estimation.
(2) Our proposed method boosts training speed by employing a two-head architecture so that the label transition matrix can be learned with a single back-propagation.
(3) Extensive experiments validate the efficacy of our proposed method in terms of both training speed and predictive performance.

%% file: Figures/1_Efficiency.tex
\begin{figure}[t] 
\centering
\subfloat{\includegraphics[width=0.255\linewidth]{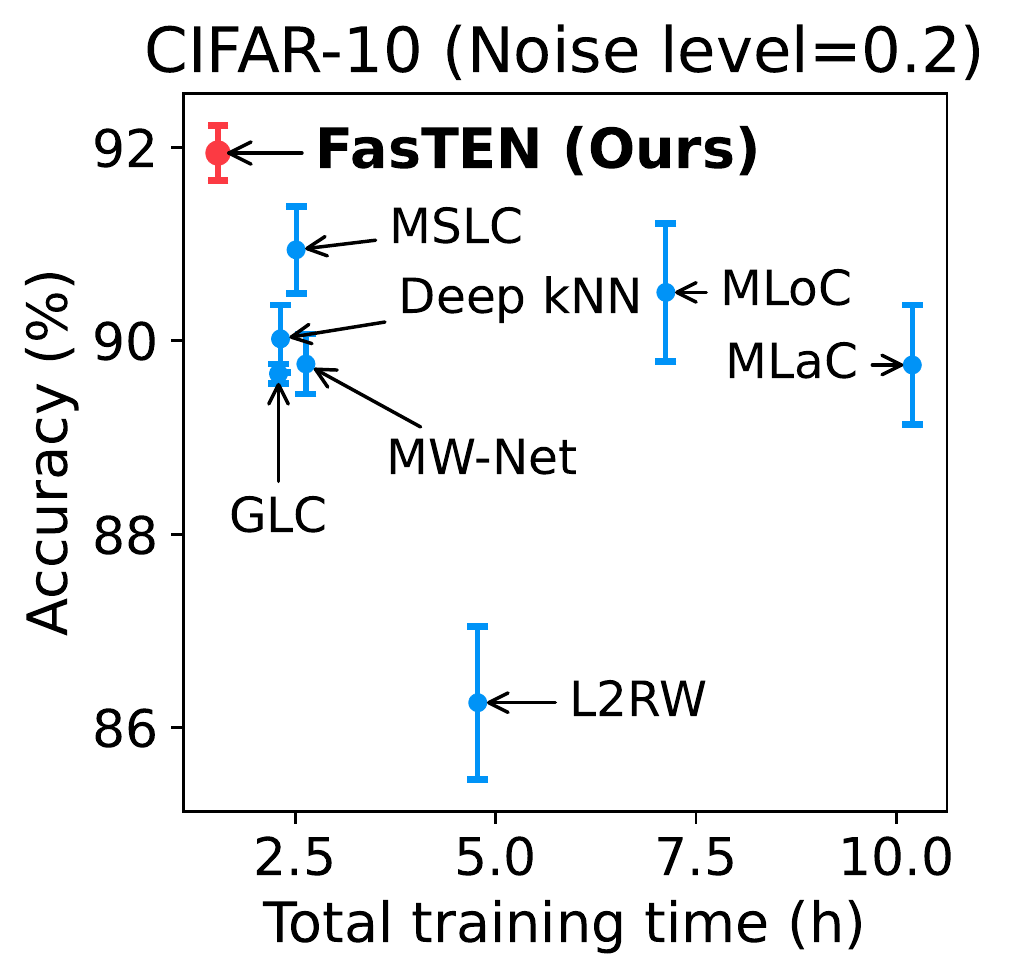}}
\subfloat{\includegraphics[width=0.245\linewidth]{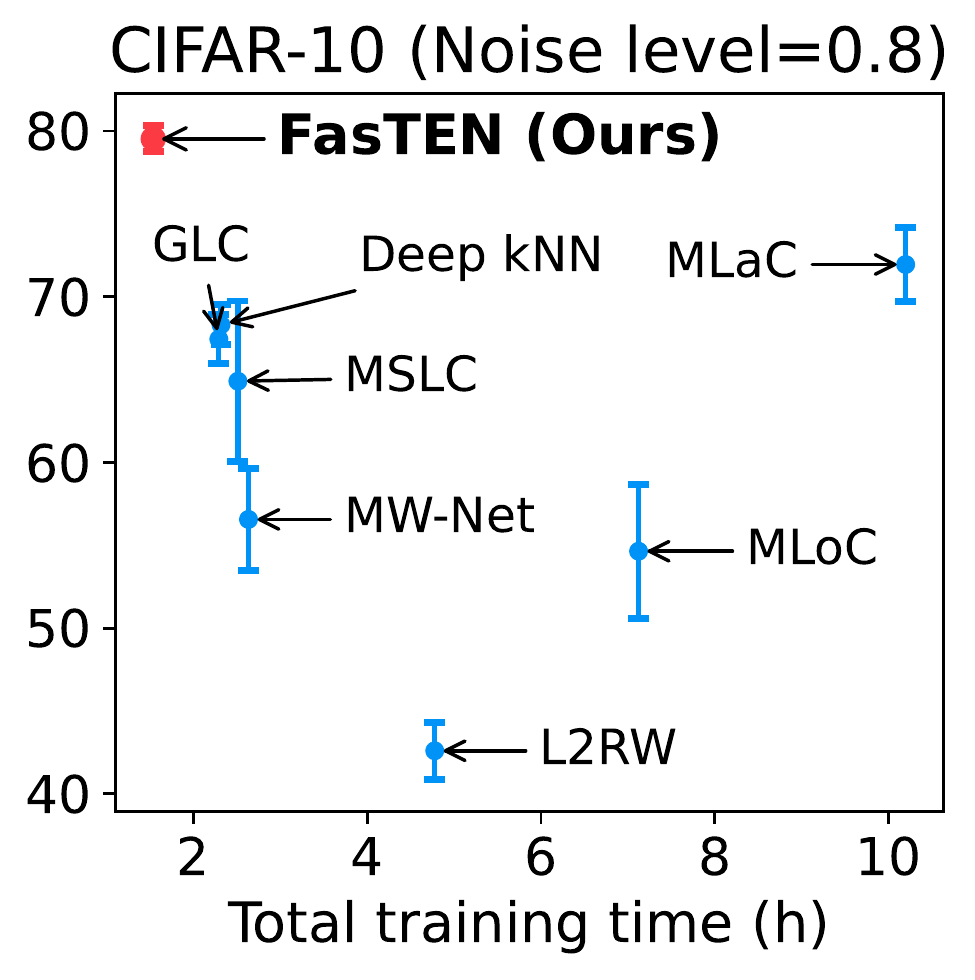}}
\subfloat{\includegraphics[width=0.25\linewidth]{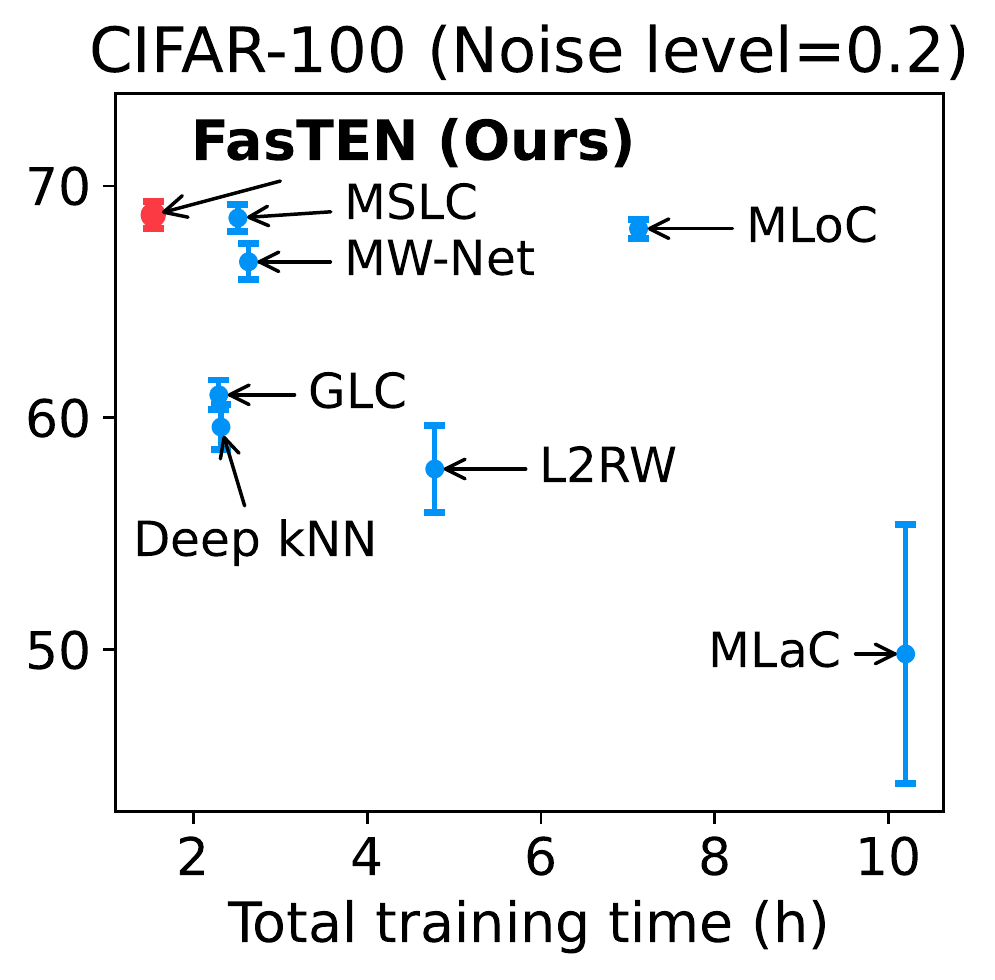}}
\subfloat{\includegraphics[width=0.25\linewidth]{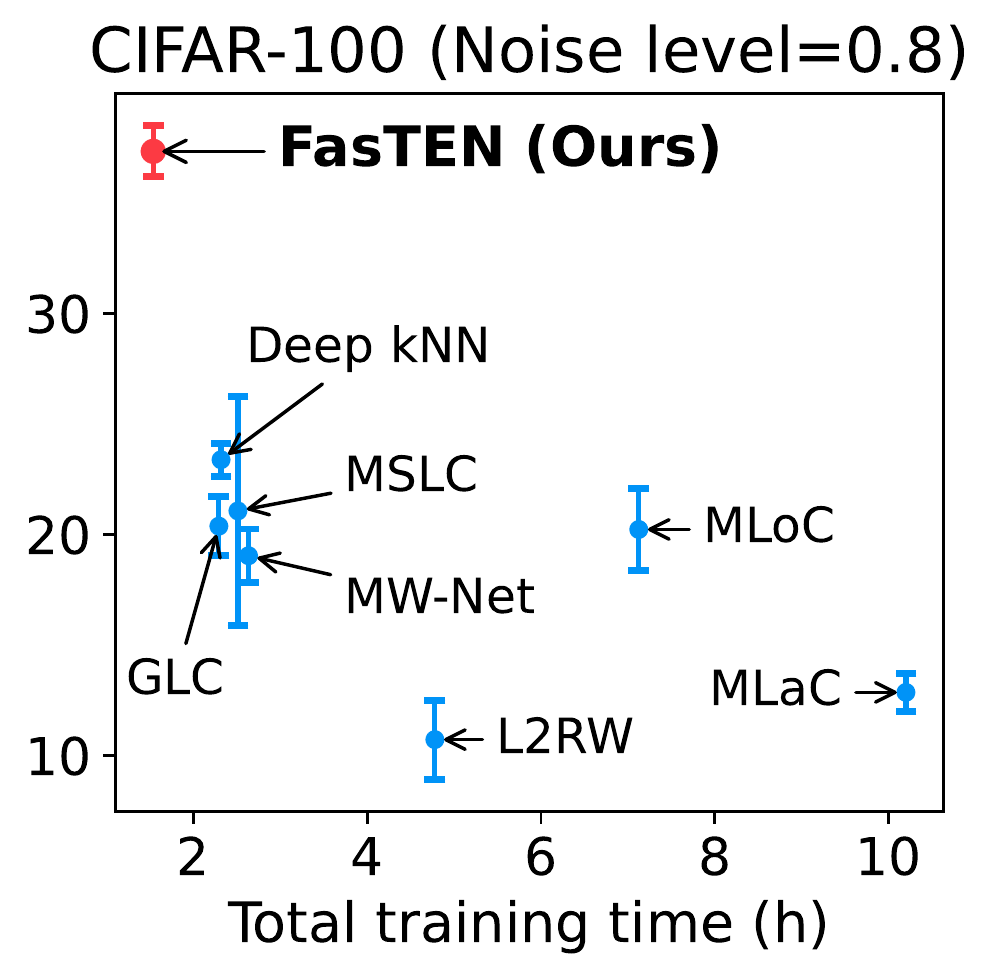}}
\caption{
Plotting accuracy (\%) \textbf{(y-axis)} according to total training time (hours) \textbf{(x-axis)}.
Our proposed method (FasTEN) shows the best performance in training efficiency while having comparable or better accuracy on both CIFAR-10/100 with various noise levels.
}
\label{fig:1_efficiency}
\end{figure}

%% file: Sections/2_Related_Work.tex
\section{Related Work}\label{sec:2_related_work}
Learning with noisy labels assumes that labels in all the training samples are potentially corrupted. 
They can be further categorized as follows: 
\textit{various loss functions}~\cite{natarajan2013learning,van2015learning,patrini2016loss,ghosh2017robust,zhang2018generalized,wang2019symmetric,van2015learning,lukasik2020does,ma2020normalized,liu2020peer,lienen2021label,yao2021instance,jiang2022an}, 
\textit{regularizations}~\cite{azadi2015auxiliary,jindal2016learning,hendrycks2019using_pre,hu2019simple,menon2020can,lukasik2020does,han2020training,song2019does,liu2020early,lienen2021label,cao2020heteroskedastic,hendrycks2019using_self,nishi2021augmentation,ortego2020multi,li2020dividemix,zhang2017mixup,harutyunyan2020improving,li2020gradient,ma2018dimensionality}, 
\textit{re-weighting training samples}~\cite{ren2018learning,jiang2018mentornet,mirzasoleiman2020coresets,liu2015classification,wang2017robust,thulasidasan2019combating,chen2019understanding,huang2020self,wang2020less,wu2020topological,wu2020topological,pleiss2020identifying}, and \textit{correcting noisy labels}~\cite{tanaka2018joint,yi2019probabilistic,han2019deep,song2019selfie,zheng2020error,guo2020ltf,kim2021fine,xia2021sample,zhang2021learning}.
However, different losses or regularizations yield inferior performance to state-of-the-art methods~\cite{zhang2017mixup,mirzasoleiman2020coresets,hu2019simple,li2020gradient}, and re-weighting methods often filter out noisy but helpful samples for extracting features to show sub-optimal performance~\cite{song2019selfie,wu2020learning,zheng2021meta,mirzasoleiman2020coresets,chang2017active,lin2017focal,shrivastava2016training}.
\textit{Label correction} methods circumvent their shortcomings by relabeling so that the feature extractor leverage the corrected labels.
However, label correction methods also have a limitation in that they are prone to propagate the error when miscorrected labels are continuously accumulated~\cite{mirzasoleiman2020coresets,wu2020learning,zheng2021meta}.
Others \textit{correct the training loss} by estimating a label transition matrix~\cite{mnih2012learning,reed2014training,sukhbaatar2014training,bekker2016training,patrini2017making,goldberger2016training,yao2019safeguarded,xia2019anchor,yao2020dual} to build a statistically consistent classifier, where the methods need multiple training stages; e.g., include a separate pretraining stage.
In this paper, we join a simple label correction method with estimating the label transition matrix to alleviate the miscorrection issue caused by miscorrected noisy labels, which only requires a single training stage.

\noindent \textbf{Learning with Noisy Label via Small Clean Dataset.}
Unlike traditional methods that use noisy datasets only~\cite{patrini2017making,yao2020dual,li2021provably}, several recent studies argue that a small clean dataset is easily obtained by techniques such as image retrieval \cite{radford2021learning}; hence one can further devise a method that effectively leverages it.
Many studies have successfully adapted the idea and shown massive performance improvement compared to the traditional methods.
Early methods~\cite{hendrycks2018using_glc,bahri2020deep,zhang2020self} require multiple training stages where it hinders the training efficiency.
Recent studies widely adopt MAML~\cite{finn2017model} to various strategies discussed above: sample re-weighting~\cite{veit2017learning,lee2018cleannet,jiang2018mentornet,ren2018learning,li2019learning,shu2019meta}, label correction~\cite{wu2020learning,zheng2021meta}, and label transition matrix estimation~\cite{wang2020training}.
These approaches first perform a virtual update with the noisy dataset, find optimal parameters using the clean dataset, and update the actual parameters by the found parameters.
This virtual update process requires three back-propagations per iteration, leading to at least three times the computational cost.
Our proposed label correction method estimates the label transition matrix using a batch drawn from the clean small dataset in a single back-propagation, greatly enhancing the training speed while showing comparable or better performance to existing state-of-the-art methods.
Additional related works are described in Appendix~\ref{appendix:additional_related_works}.

%% file: Sections/3_Method.tex
\section{Methodology}\label{sec:3_methodology}

Existing label correction methods try to find and fix noisy labels to utilize them as clean samples in model training, where they can improve the classification performance by reducing the noise level of the whole training samples.
However, erroneously corrected samples, i.e., clean samples deemed noisy, or vice versa, are often kept throughout the model training.
Since current label correction methods blindly trust these miscorrected labels, this behavior degrades the classification performance under the noisy label situation (\S~\ref{subsec:4_4_robustness_to_miscorrected_labels}).

In this section, we show that the accurately estimated label transition matrix with the clean dataset alleviates the miscorrection problem of existing label correction methods. 
Further, we describe our efficient method estimating the label transition matrix for every training iteration while correcting noise labels. Our proposed method is illustrated in Figure~\ref{fig:2_model_architecture} and summarized in Algorithm~\ref{alg:1_lt2l}.

\input{Figures/2_Model_Architecture}
\subsection{Batch Formation}\label{subsec:3_1_episodic_batch_formation}
We estimate the transition matrix to track the shifted noisy label distribution caused by label correction using a clean batch.
To ensure the effective estimation of the label transition matrix, we formulate the batch to have the same number of samples per class.
We first compose the clean batch $d$ with randomly chosen $K$ samples for the entire $N$ classes in the clean dataset $\mathcal{D}$ to get benefits from having a certain amount of clean samples for each class, as follows $d = \{(x_n,y_n)\}_{n=1}^{KN}$ where $x$ is an input and $y \in \mathbb{R}^N$ is the clean label of $x$.
A noisy batch $\Bar{d}$ is randomly sampled from the noisy dataset $\Bar{\mathcal{D}}$, as follows $\Bar{d} = \{(\Bar{x}_n,\Bar{y}_n)\}_{n=1}^{M}$ where $\Bar{y} \in \mathbb{R}^N$ is the noisy label of $\Bar{x}$ and $M$ is the size of the noisy batch which we set as $M=KN$ for simplicity.
It is different from other methods~\cite{zheng2021meta,wu2020learning,shu2019meta,ren2018learning} based on meta-learning which randomly compose the clean batch.

\subsection{Transition Matrix Estimation}\label{subsec:3_2_transition_matrix_estimation}
\input{Algorithms/1_learning_transition_matrix_to_learn}
Each element $\emT_{ij}$ of the label transition matrix $\mT \in \mathbb{R}^{N\times N}$ is defined as the probability of a clean label $i$ to be corrupted as a noisy label $j$, i.e. $\emT_{ij} = p ( \Bar{y} = j | y = i )$.
It is well-known that a robust classifier can be obtained with the accurately estimated label transition matrix~\cite{sukhbaatar2014training,patrini2017making,hendrycks2018using_glc,xia2019anchor,yao2020dual}.
We choose a simple but accurate method that directly estimates the posterior with a clean dataset~\cite{hendrycks2018using_glc,xia2019anchor,yao2020dual}, whereas there are other more sophisticated methods that estimate the label transition matrix~\cite{mnih2012learning,reed2014training,sukhbaatar2014training,bekker2016training,goldberger2016training,patrini2017making}.
Following the assumption of the previous work~\cite{hendrycks2018using_glc,chen2020robustness,berthon2020confidence,xia2020part}, we also assume conditional independence of $\Bar{y}$ and $y$ given $x$:
\begin{equation}
\begin{split}
    p ( \Bar{y} &| y)  = p ( \Bar{y} | y ) \int p ( x | \Bar{y} , y ) dx = \int p (\Bar{y} | y, x) p (x|y) dx = \int p (\Bar{y} | x) p (x|y) dx.
\end{split}
\end{equation}
We design the transition matrix to be \textit{class-dependent}, i.e., $p ( \Bar{y} | y, x ) = p ( \Bar{y} | y)$, following recent state-of-the-art methods~\cite{liu2015classification,scott2015rate,xia2019anchor,yao2020dual}.
By parameterizing a feature extractor $\Bar{\phi}$ and a linear classifier $\Bar{\theta}$, we obtain $p (\Bar{y} | x) = f_{\Bar{\phi}, \Bar{\theta}} (x)$ where $f_{\Bar{\phi}, \Bar{\theta}}$ is the noisy classifier that consists of the linear classifier and the feature extractor trained only with the noisy labels.
If the noisy classifier $f_{\Bar{\phi}, \Bar{\theta}}$ gives a perfect prediction for the noisy data, we can estimate the transition probability $p ( \Bar{y} | y )$ using the clean samples $(x, y) \in d$ as follows (See Appendix~\ref{appendix:theory_t} for details):
\begin{equation}\label{eq:estimation_t}
    \widehat{\mT} \gets ( \sum_{(x,y)\in d} y f_{\Bar{\phi}, \Bar{\theta}} (x)^\top )  \text{diag}{}^{-1} ( \sum_{(x, y) \in d} y ).
\end{equation}
We emphasize the importance of the transition matrix estimation, as its accuracy determines the bounds of the generalization error of the classifier~\cite{xia2019anchor}.
However, the limited number of clean samples inside a single batch may yield an inaccurate transition matrix, even with the ideal $f_{\Bar{\phi}, \Bar{\theta}}$.
We analyze the upper bound of the estimation error as follows:
\begin{theorem}\label{theorem:transition_matrix_error}
Assume the Frobenius norm of the weight matrices $\Bar{\phi}_1, ..., \Bar{\phi}_{H-1}, \Bar{\theta}$ are at most $\Bar{\Phi}_1, ..., \Bar{\Phi}_{H-1}, \Bar{\Theta}$ for $H$-layer neural networks $f_{\Bar{\phi}, \Bar{\theta}}$.
Let the loss function be $L$-Lipschitz continuous w.r.t. $f_{\Bar{\phi}, \Bar{\theta}}$.
Let the activation functions be 1-Lipschitz, positive-homogeneous, and applied element-wise (such as ReLU). 
Let $x$ be upper bounded by B, i.e., for any $x \in \mathcal{X}$, $\|x\| \le B$. Then, for $\epsilon \ge 0$
\begin{equation}
\begin{split}
    p \left( \left| \widehat{\emT}_{ij} - \emT_{ij} \right| > \epsilon \right) \le \frac{NLB(\sqrt{ 2H \log 2} + 1) \Bar{\Theta} \Pi_{h=1}^{H-1} \Bar{\Phi}_i}{\sqrt{ | \Bar{\mathcal{D}} |  }} \\
    + \frac{\sqrt{ -\log ( \epsilon )} }{\sqrt{2 | \Bar{\mathcal{D}} |  }} + 2 \exp \left( -2 \epsilon^2 K  \right).
\end{split}
\end{equation}
\end{theorem}


\begin{proof}
See Appendix~\ref{appendix:theory_proof}.
\end{proof}

Although the upper bound of the estimation error of the transition matrix is affected by the batch size $K$, we empirically verify that small $K$ does not necessarily harm the classification performance (See Appendix~\ref{subsubsec:4_4_2_how_sensitive_is_to_the_batch_size}).

\subsection{Learning with Estimated Transition Matrix}\label{subsec:3_3_learing_with_estimated_transition_matrix}
A clean classifier $f_{\phi,\theta}$ is trained with the estimated transition matrix $\widehat{\mT}$:
\begin{equation}\label{eq:clean}
\small
    \mathcal{L}_\text{clean} = \sum_{(x,y)\in d} \mathcal{L}_\text{CE} \left(f_{\phi,\theta} (x), y \right) + \sum_{(\Bar{x}, \Bar{y}) \in \Bar{d}} \mathcal{L}_\text{CE} \left(   \widehat{\mT}^\top f_{\phi,\theta} (\Bar{x}), \Bar{y} \right),
\end{equation}
given the cross-entropy loss function $\mathcal{L}_\text{CE}$, where the feature extractor $\phi$ and the linear classifier $\theta$ form the clean classifier $f_{\phi,\theta}$ which estimates clean labels.
If $\widehat{\mT}$ is correctly estimated, the clean classifier $f_{\phi, \theta}$ becomes statistically consistent~\cite{sukhbaatar2014training,patrini2017making,hendrycks2018using_glc,xia2019anchor,yao2020dual}.
This approach makes the clean classifier skeptical towards corrected labels, hence avoiding the miscorrection issue.

On the other hand, the noisy classifier $f_{\Bar{\phi}, \Bar{\theta}}$ is trained to model the noisy label distribution.
\begin{equation}
    \mathcal{L}_\text{noisy} = \sum_{(\Bar{x}, \Bar{y}) \in \Bar{d}} \mathcal{L}_\text{CE} \left(f_{{\Bar{\phi},\Bar{\theta}}} (\Bar{x}), \Bar{y} \right) 
\end{equation}
We emphasize that updating the noisy classifier $f_{\Bar{\phi}, \Bar{\theta}}$ every iteration is critical as it can adaptively model the ever-changing noisy label distribution on the fly, where the distribution constantly shifts as the noisy labels are actively corrected to reduce the noise level (See \S~\ref{subsec:3_4_method_relabel}).

\subsection{Efficient Training}\label{subsec:3_4_efficient_training}
Similar to \cite{vinyals2016matching,jiang2018mentornet}, we propose an efficient training scheme through weight sharing via two-head architecture, as shown in Figure~\ref{fig:2_model_architecture}.
Where the architecture closely resembles the ones of \cite{vinyals2016matching,jiang2018mentornet}, our two-head architecture only shares the feature extractor $\phi = \Bar{\phi}$.
Unlike the shared feature extractor, our architecture does not share the linear classifier since modeling both noisy and clean data distribution with a single linear classifier is impractical.
Based on the two-head architecture, the given samples require only a single inference on the feature extractor for (1) training classifiers, (2) estimating the transition matrix, and (3) correcting labels, which makes model training highly efficient.
Thus, we define the clean and noisy classifier as $f_{\phi, \theta}$ and $f_{\phi, \bar{\theta}}$, respectively, to produce our final objective function $\mathcal{L}$:
\begin{equation}\label{eq:final_loss}
\mathcal{L} = \mathcal{L}_\text{clean} + \lambda \mathcal{L}_\text{noisy}
\end{equation}
where $\lambda$ is a loss balancing factor. 
In order to prevent over-fitting on $\Bar{d}$, we introduce $\lambda$ to the final objective function. 
We search for the optimal hyperparameter $\lambda$ for all of our experiments (See Appendix~\ref{appendix:searching_the_optimal_hyperparameter}).

\noindent \textbf{Efficiency Analysis.}
Compared to the vanilla training scheme, which assumes that all labels are clean, we only add a single linear classifier $\Bar{\theta}$ with only $N$ additional parameters.
Also, our loss only requires a single back-propagation, where the added linear classifier has a negligible computational burden.
Our training scheme stands out even more compared to the existing MAML-based methods~\cite{wu2020learning,zheng2021meta} or multi-stage training~\cite{hendrycks2018using_glc,bahri2020deep} (See $\S$~\ref{subsec:4_2_training_time_comparison} and Figure~\ref{fig:1_efficiency}).

\subsection{Label Correction}\label{subsec:3_4_method_relabel}
In this paper, we focus on the efficient, on-the-fly estimation of the label transition matrix to combat label miscorrection.
To further demonstrate the effectiveness of our method, we employ a naïve label correction strategy where we feed each noisy set sample $x \in \Bar{d}$ to the clean classifier $f_{\phi, \theta}$ to produce a probability vector.
If the maximum probability $\max(f_{\phi, \theta}(x))$ is bigger than the threshold $\rho$, we correct its label to a more probable label.
This strategy relies only on the most recent prediction of the model mid-training, so the decision is prone to change.
Formally, we can describe the relabeled $\hat y$,
\begin{equation}
\small
    \hat y = \begin{cases}
            \Bar{y}^*,& \text{if } \max (f_{\phi,\theta} (\Bar{x})) < \rho\\
            \lfloor f_{\phi,\theta} (\Bar{x}) / \max (f_{\phi,\theta} (\Bar{x})) \rfloor,              & \text{otherwise}
        \end{cases}
\end{equation}
where $\lfloor \cdot \rfloor$ denotes floor function and $\Bar{y}^*$ denotes the original label from $\bar{d}$.
$\Bar{y}^*$ differs from $\Bar{y}$; the former denotes the original label from the noisy dataset, whereas the latter is continuously corrected by the above strategy.
Even with this simple strategy, our model shows better performance compared to the state-of-the-art methods.
The experimental results suggest that replacing this strategy may further improve the model performance.

%% file: Figures/2_Model_Architecture.tex
\begin{figure}[t] 
\centering
\includegraphics[width=\columnwidth]{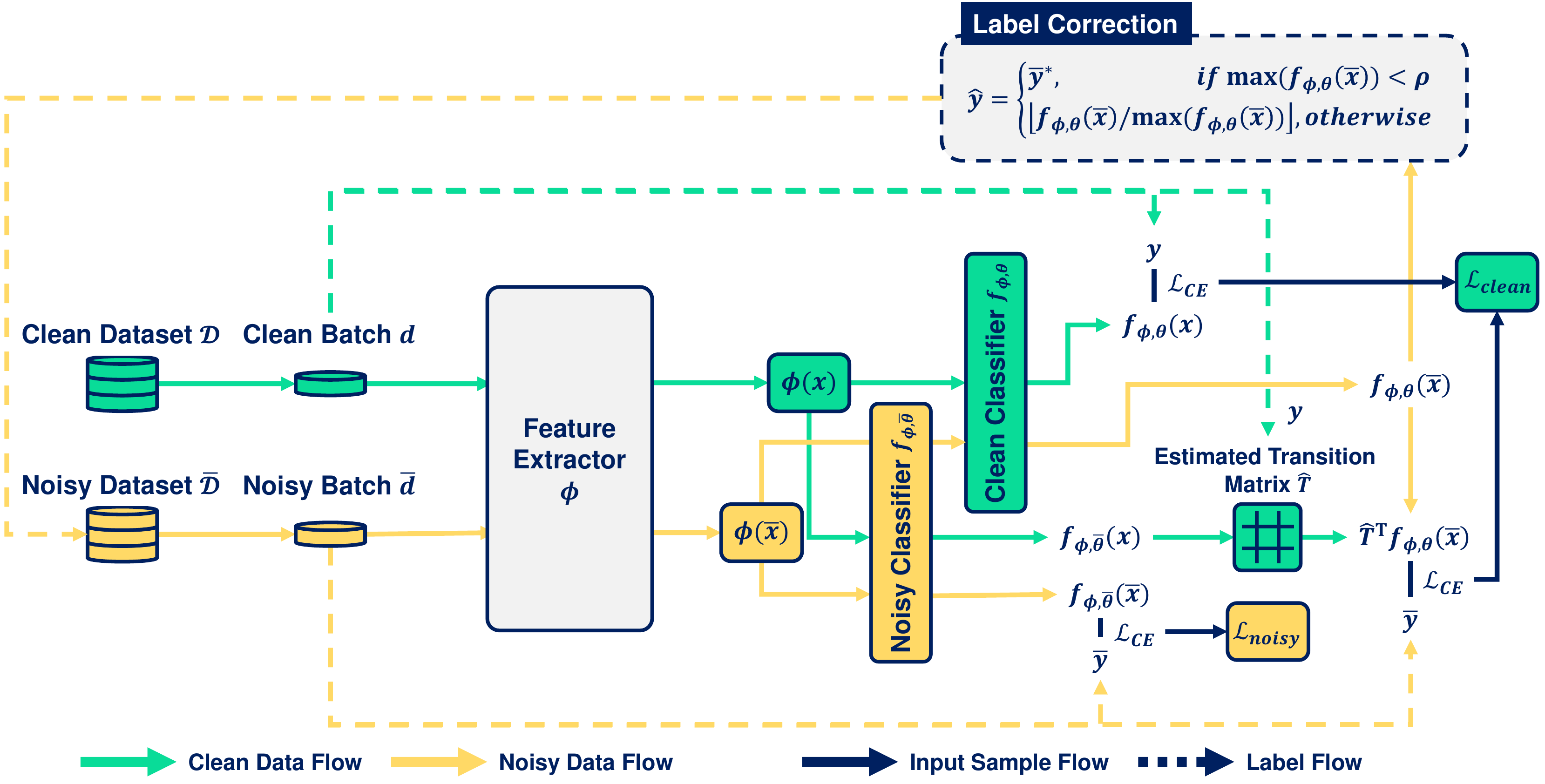}
\caption{
Summarization of our proposed method (FasTEN).
}
\label{fig:2_model_architecture}
\end{figure}

%% file: Algorithms/1_learning_transition_matrix_to_learn.tex
\begin{algorithm}[t]

\caption{\small\textbf{Fas}t \textbf{T}ransition Matrix \textbf{E}stimation for Learning with \textbf{N}oisy Labels (\textbf{FasTEN})}
\label{alg:1_lt2l}{ 
\begin{algorithmic}
\scriptsize
    \State {\bfseries Input:} Clean dataset $\mathcal{D}$, noisy dataset $\Bar{\mathcal{D}}$.
    \State {\bfseries Hyper-parameters:} Label correction threshold $\rho$, Controllable loss ratio for noisy classifier $\lambda$.
    \State {\bfseries Output:} Clean classifier $f_{\phi, \theta}$ where linear classifier $\theta$ and feature extractor $\phi$.
    \State{Randomly initialize common feature extractor $\phi$, linear classifiers $\theta$, and $\Bar{\theta}$ for clean labels and noisy labels.}
    \For{each epoch $i=0, \cdots$}
        \For{each iteration in epoch $i$}
            \State{Sample mini-batch $d \sim \mathcal{D}$, $\Bar{d} \sim \Bar{\mathcal{D}}$.}
            \State{$\Hat{T} \gets \left( \sum\limits_{(x,y)\in d} y f_{\phi, \Bar{\theta}} (x)^\top \right) \text{diag}^{-1} \left( \sum\limits_{(x, y) \in d} y \right) $}
            \State{$\mathcal{L}_\text{clean} \gets \sum\limits_{(x,y)\in d}  \mathcal{L}_\text{CE} \left(f_{\phi,\theta} (x), y \right) + \sum\limits_{(\Bar{x}, \Bar{y}) \in \Bar{d}}  \mathcal{L}_\text{CE} \left( \Hat{T}^\top f_{\phi,\theta} (\Bar{x}), \Bar{y} \right)$}
            \State{$\mathcal{L}_\text{noisy} \gets \sum_{(\Bar{x}, \Bar{y}) \in \Bar{d}} \mathcal{L}_\text{CE} \left(f_{{\phi,\Bar{\theta}}} (\Bar{x}), \Bar{y} \right) )$}
            \State{
            $\Bar{\mathcal{D}} \gets \left( \Bar{\mathcal{D}} - \Bar{d}\right) \cup \left\{ \left(\Bar{x}, \begin{cases}
                \Bar{y}^* \scriptstyle{\text{, if } \max (f_{\phi,\theta} (\Bar{x})) < \rho}\\
                \lfloor  f_{\phi,\theta} (\Bar{x}) / {\max (f_{\phi,\theta} (\Bar{x})) }\rfloor               \scriptstyle{\text{, otherwise}}
            \end{cases}\right) \bigg\rvert  (\Bar{x}, \Bar{y}) \in \Bar{d} \right\} $
            }
            \State{Update $\phi, \theta, \Bar{\theta}$ using $\nabla_{\phi, \theta, \Bar{\theta}} \ (\mathcal{L}_\text{clean} + \lambda \mathcal{L}_\text{noisy})$ with a single back-propagation.}
        \EndFor
    \EndFor

\end{algorithmic}
}
\end{algorithm}

%% file: Sections/4_Experiments.tex
\section{Experiments}\label{sec:4_experiments}
In this section, we evaluate our proposed learning method, FasTEN, in terms of predictive performance ($\S$~\ref{subsec:4_1_predictive_performance_comparison}) and efficiency ($\S$~\ref{subsec:4_2_training_time_comparison}).
We also validate the label correction performance to demonstrate that our method is better in correcting noisy labels ($\S$~\ref{subsec:4_3_relabel_performance} and Appendix~\ref{appendix:exp_detection}) and experimentally show the robustness of our proposed method towards miscorrected labels. ($\S$~\ref{subsec:4_4_robustness_to_miscorrected_labels}).
We further analyze whether our method successfully estimates the label transition matrix in the case where the label correction shifts the true label transition matrix ($\S$~\ref{subsec:4_5_effect_of_on_the_fly}) or not ($\S$~\ref{subsec:4_6_empirical_convergence_analysis}).
Additional experimental results and further analyses are described in Appendix~\ref{appendix:additional_experiments}.
We provide the source codes\footnote{\url{https://github.com/hyperconnect/FasTEN}} for the reproduction of the experiments conducted in this paper.

\input{Tables/1_Evaluation_Results_on_Syntactic}
\noindent \textbf{Baselines using the small clean dataset.} We deliberately choose the baselines that utilize the small clean dataset in learning with noisy labels.
These baselines are categorized in the following three types.
\textit{Re-weighting}: \textbf{L2RW}~\cite{ren2018learning} learns to assign weights to training samples based on their gradients.
\textbf{MW-Net}~\cite{shu2019meta} trains an explicit weighting function with the training samples. 
\textbf{Deep kNN}~\cite{bahri2020deep} applies the k-nearest neighbor algorithm to the logit layer of classifiers to find noisy samples.
\textit{Label transition matrix estimation}: \textbf{GLC}~\cite{hendrycks2018using_glc} estimates the label transition matrix using the small clean dataset.
\textbf{MLoC}~\cite{wang2020training} considers the label transition matrix as trainable parameters to be obtained through meta-learning.
\textit{Label correction}: 
\textbf{MLaC}~\cite{zheng2021meta} trains a label correction network as a meta-process to provide corrected labels.
\textbf{MSLC}~\cite{wu2020learning} uses soft
labels with loss balancing weight through meta-gradient descent step under the guidance of the clean dataset.


\subsection{Predictive Performance Comparison }\label{subsec:4_1_predictive_performance_comparison}

\noindent \textbf{CIFAR-10/100 with Synthetic Noise.} CIFAR-10/100~\cite{krizhevsky2009learning} have been widely adopted to assess the robustness of the methods to noisy labels.
Since CIFAR-10/100 are known as clean datasets, labels are synthetically manipulated to contain noisy labels, injecting two types of noise: symmetric and asymmetric.
\textbf{Symmetric:} The labels are randomly flipped with uniform distribution. \textbf{Asymmetric:} the labels are flipped with class-dependent distribution, following the evaluation protocol of \cite{patrini2017making,yao2019safeguarded}.
We claim that most studies report the performance highly overfitted to the test set without hyperparameter tuning on the validation set~\cite{wu2020learning,li2020dividemix,nishi2021augmentation,ortego2020multi}.
Moreover, baseline models employ different backbone networks, making it challenging to dissect the performance improvement whether it originated from each method or the backbone networks.
Therefore, we first extract 5K samples as the validation set from the training set containing 50K samples and further extract 1K samples as the clean dataset.
Then, we unify the backbone network as ResNet-34~\cite{he2016deep}, which is widely adopted in various baselines~\cite{wu2020learning,liu2020early}.
Note that we do our best to maintain the experimental settings of each method, including the hyperparameters written in the original paper.
Detailed settings are deferred to Appendix~\ref{appendix:experimental_details}. 

\noindent \textbf{Results.} Table~\ref{tab:1_performance_synthetic} summarizes the evaluation results on CIFAR-10/100.
For both CIFAR-10/-100, our proposed FasTEN achieves state-of-the-art performance on various noise levels within 95\% confidence intervals.
Especially, under a high noise level (80\%), our FasTEN considerably outperforms the baselines with small variance on performance, which implies the robustness of our method~\cite{li2016low,li2017mixture}.
These results demonstrate that our proposed method performs well in learning with noisy labels, especially considering its training efficiency (See $\S$~\ref{subsec:4_2_training_time_comparison}).

\input{Tables/2_Evaluation_Results_on_Realworld}
\noindent \textbf{Clothing1M with Real-world Noise.}
Clothing1M~\cite{xiao2015learning} is a noisy real-world dataset that consists of one million samples with additional 47K human-annotated clean samples.
We use its original splits of clean and noisy data.
For a fair comparison, we employ ResNet-50 architecture pretrained with the ImageNet dataset~\cite{deng2009imagenet} for the initial backbone architecture.
Evaluation results on Clothing1M are summarized in Table~\ref{tab:2_performance_real}.

\noindent \textbf{Further baselines.}
We further compare our proposed FasTEN with additional baselines that have already reported their performance on Clothing1M dataset.
Since the data split of Clothing1M dataset is the same for all the baselines, we simply obtain the performance of the baselines from their original papers and report the performance in Table \ref{tab:2_performance_real}.
\textbf{DivideMix}~\cite{li2020dividemix} and \textbf{AugDesc}~\cite{nishi2021augmentation} leverages semi-supervised learning with various data augmentation strategies.
\textbf{Forward}~\cite{patrini2017making}, \textbf{T-Revision}~\cite{xia2019anchor}, \textbf{IF}~\cite{jiang2022an}, and \textbf{causalNL}~\cite{yao2021instance}, and \textbf{VolMinNet}~\cite{li2021provably} are transition matrix estimation methods that use certain data points without clean data points.

\noindent \textbf{Results.}
As shown in Table~\ref{tab:2_performance_real}, our proposed FasTEN achieves remarkable performance on Clothing1M which contains instance-dependent noisy labels, beating the baselines by a large margin.
This evaluation result indicates that our proposed FasTEN is more applicable in real-world problems where label corruption frequently occurs, although it does not directly target to address the problem of instance-dependent noisy labels.
Similar to previous observations \cite{menon2016learning,hendrycks2018using_glc}, we suspect that using the transition matrix seems to combat instance-dependent noise to some extent.
Also, not only that our method shows superior performance over all the baselines that use the small clean set, but it also surpasses the semi-supervised learning-based methods (DivideMix and AugDesc) without any complex augmentation techniques.
Finally, FasTEN shows better performance than T-Revision, causalNL, IF, and VolMinNet, which estimate the transition matrix without the small clean data (this is not a fair comparison).
This result indicates that using the small clean data is effective in estimating the transition matrix accurately, leading to performance improvement eventually.

\subsection{Training Time Comparison}\label{subsec:4_2_training_time_comparison}
\input{Tables/7_Training_Time_Comparison}

\noindent \textbf{Setup.}
To verify the efficiency of our proposed FasTEN, we compare it with the baselines in terms of accuracy by total training time.
Total training time is measured on CIFAR-10/-100, respectively, with a single RTX 2080Ti GPU.
Test accuracy shows the predictive performance on CIFAR-10/-100 with 20\% and 80\% symmetric noise ratios, the mildest and most severe noise conditions, respectively.
Since Deep kNN and GLC require multiple training stages, the summation of all the hours needed for each training phase is provided.

\noindent \textbf{Results.}
Figure~\ref{fig:1_efficiency} shows that our FasTEN, which learns the label transition matrix with the single back-propagation in the single-training stage, makes model training more efficient than other baselines that need multiple back-propagation or multiple training stages while showing better performance.
Table~\ref{table:3_appendix_efficiency} shows the total training hours of each baseline, including our FasTEN.
Our method provides the training time speedup of minimum $\times$1.49 to maximum $\times$6.64.

\subsection{Label Correction Performance Comparison}
\label{subsec:4_3_relabel_performance}
\input{Tables/3_Evaluation_Results_on_Relabel}
We analyze the predictive performance of the baseline methods on all the training samples (Overall) and the wrongly labeled subset of them (Incorrect), respectively.
Table~\ref{tab:3_performance_relabel} demonstrates that our method can successfully correct the noisy labels, where using the label correction further improves the correction performance.
This also implies that our FasTEN may be helpful in further cleansing the noisy training set.

We also compare the performance between Overall and Incorrect cases.
Re-weighting (L2RW, MW-Net, Deep kNN) and transition matrix estimation-based methods (GLC, MLoC) show similar performance between two cases: Overall and Incorrect.
However, the performance of the meta-model of MLaC is worse for the Incorrect case, which indicates that the correction from the meta-model is less effective where the labels are wrong.
Also, notable underperformance of the meta-model of MSLC may indicate the inefficacy of the meta-model.
We also analyze the meta-model of the re-weighting methods in the Appendix~\ref{appendix:exp_detection}, where they do not distinguish the wrongly labeled samples well.

\input{Figures/2_3_Robustness_and_Transition_Matrix}
\subsection{Robustness to Miscorrection: What Happens if Labels are Wrongly Corrected?}\label{subsec:4_4_robustness_to_miscorrected_labels}
This subsection illustrates the robustness of our label correction method to miscorrected labels by comparing it with other label correction methods (MLaC and MSLC) which blindly trust the miscorrected labels as the ground-truth, where we verify the imperfect corrections (See $\S$~\ref{subsec:4_3_relabel_performance}).
We examine how much this behavior deteriorates the predictive performance.

\noindent \textbf{Setup.}
We experiment on CIFAR-10 with symmetric 80\% noise where there are a maximum number of noisy labels to correct. 
To simulate the miscorrection, we perturb the corrected labels by injecting artificial noise.
We control the degree of random perturbation to observe the robustness of each method on various levels of miscorrection.
We further assess the robustness of our FasTEN and MSLC by comparing it with the performance obtained without label correction.

\noindent \textbf{Results.}
Figure~\ref{fig:2_robustness} shows our proposed FasTEN outperforms MLaC and MSLC on all the degrees of the random perturbation.
MLaC shows steep performance degradation when perturbation worsens, i.e., there are more miscorrected labels.
This observation reveals the susceptibility of MLaC.
MSLC shows trivial performance gains when labels are corrected, implying that it is not using the full benefits of label correction.
Furthermore, when highly perturbed, MSLC performance worsens if it attempts to correct the labels.
In contrast, the label correction of our FasTEN improves performance even in harsh situations.
FasTEN does not degrade performance even if the correction becomes useless (100\% perturbation).
These observations show that our FasTEN builds a more robust classifier to miscorrected labels through its efficient estimation of the label transition matrix, acting as a safeguard combating the miscorrected labels.

\subsection{On-the-fly Estimation of the Label Transition Matrix}\label{subsec:4_5_effect_of_on_the_fly}
Our proposed FasTEN newly estimates the label transition matrix on every iteration, where the matrix is constantly shifted by label correction.
To assess the matrix estimation quality, we compare it with the true label transition matrix.

\noindent \textbf{Setup.}
We train FasTEN on CIFAR-10 with symmetric 80\% and asymmetric 40\% noise, which are harsh conditions on symmetric and asymmetric noise injection, respectively.
We compare the estimated label transition matrix $\widehat{\mT}$ with the true label transition matrix $\mT$ by observing the mean of diagonal term values for each epoch.
The mean of the diagonal term in the transition matrix represents the average of the probability that a sample is mapped to a clean label.

\noindent \textbf{Results.}
Figure~\ref{fig:3_transition_matrix} shows the overall tendency of the estimated transition matrix (red) to follow the true label matrix (blue).
In the asymmetric 40\% setting, diagonal term values of the true label transition matrix $\mT$ gradually increases (blue), which indicates the dataset is cleansed by the label correction.
However, in the symmetric 80\% case, diagonal term values of the true transition matrix $\mT$ decreases at the middle of the training.
As we maintain the fixed threshold $\rho$, the total number of corrected samples decreases.
Nonetheless, we can conclude that the transition matrix is successfully estimated on shifting noise levels.

Additionally, we observe that the estimated transition matrix $\widehat{\mT}$ shows higher mean values, i.e., being overconfident on the clean dataset samples.
Theoretically, $f_{\phi, \Bar{\theta}}$ should correctly approximate the noisy label distribution given enough number of clean samples (See Appendix~\ref{appendix:theory_proof}), but it seems to be overfitting to the clean dataset in practice.
This observation is consistent with the popular belief that neural networks tend to learn clean samples first and noisy samples later~\cite{arpit2017closer}.
For better matrix estimation to yield a more robust classifier~\cite{han2020sigua,xia2019anchor,mirzasoleiman2020coresets,yao2020dual}, it appears that we need to address the overfitting through additional components.

\subsection{Empirical Convergence Analysis on Estimating the Label Transition Matrix}
\label{subsec:4_6_empirical_convergence_analysis}

\noindent \textbf{Setup.}
This section analyzes the convergence of estimation error between the true label transition matrix $\mT$ and the estimated transition matrix $\widehat{\mT}$, comparing our FasTEN to other methods, MLoC and GLC, which learn the transition matrix.
For fair comparison, we exclude the label correction for our method.

\noindent \textbf{Results.}
Figure~\ref{fig:convergence} shows the difference between the probability distribution of the true label transition matrix $\mT$ and the estimated transition matrix $\widehat{\mT}$ for each iteration, where Pearson $\chi^2$-divergence is used to measure the discrepancy between the two matrices.
GLC error remains fixed (dotted line) as it estimates the transition matrix only once in the entire learning process.
The decrease of MLoC error is extremely slow (blue line), implying the high dependence of the initialization of $\widehat{\mT}$ and its ineffectiveness on estimation.
Although our FasTEN does not require multiple stages of training and produces the single mini-batch-based estimate every iteration, it shows fast convergence with a similar estimation error to GLC, which uses all the available data.
\input{Figures/6_Convergence}

%% file: Tables/1_Evaluation_Results_on_Syntactic.tex
\begin{table*}[t]
\caption{
Performance comparison on CIFAR-10/100 datasets under various noise level. Test accuracy (\%) with 95\% confidence interval of 5-runs is provided.}
\label{tab:1_performance_synthetic}
\small
\centering
\resizebox{\textwidth}{!}{%
\begin{tabular}{c|c|cccc|cc}
\toprule
\multirow{2}{*}{} & \multirow{2}{*}{Method} & \multicolumn{4}{c|}{Symmetric Noise Level} & \multicolumn{2}{c}{Asymmetric Noise Level} \\ 
& & 20\% & 40\% & 60\% & 80\% & 20\% & 40\% \\
\midrule
\multirow{8}{*}{\STAB{\rotatebox[origin=c]{90}{CIFAR-10}}} 
& L2RW & 88.26 \footnotesize $\pm$ 0.79 & 83.76 {\footnotesize $\pm$ 0.54} & 74.54 {\footnotesize $\pm$ 1.54} & 42.60 {\footnotesize $\pm$ 1.71} & 88.79 {\footnotesize $\pm$ 0.63} & 85.86 {\footnotesize $\pm$ 0.87} \\
& MW-Net & 89.76 {\footnotesize $\pm$ 0.31}  & 86.52 {\footnotesize $\pm$ 0.28}  & 81.68 {\footnotesize $\pm$ 0.25}  & 56.56 {\footnotesize $\pm$ 3.07}  & 91.31 {\footnotesize $\pm$ 0.25}  & 88.69 {\footnotesize $\pm$ 0.37}  \\
& Deep kNN & 90.02 {\footnotesize $\pm$ 0.35} & 87.27 {\footnotesize $\pm$ 0.39} & 82.80 {\footnotesize $\pm$ 0.55} & 68.30 {\footnotesize $\pm$ 1.21} & 89.97 {\footnotesize $\pm$ 0.48}  & 84.56 {\footnotesize $\pm$ 0.87} \\
\cmidrule{2-8} 
& GLC & 89.66 {\footnotesize $\pm$ 0.10}  & 85.30 {\footnotesize $\pm$ 0.73} & 80.34 {\footnotesize $\pm$ 0.73} & 67.44 {\footnotesize $\pm$ 1.50} & 91.56 {\footnotesize $\pm$ 0.66} & \textbf{89.76 {\footnotesize $\pm$ 0.89}} \\
& MLoC & 90.50 {\footnotesize $\pm$ 0.71} & 87.20 {\footnotesize $\pm$ 0.35} & 81.95 {\footnotesize $\pm$ 0.44} & 54.64 {\footnotesize $\pm$ 4.04} & 91.15 {\footnotesize $\pm$ 0.16} & 89.35 {\footnotesize $\pm$ 0.45} \\
\cmidrule{2-8} 
& MLaC & 89.75 {\footnotesize $\pm$ 0.62}  & 86.63 {\footnotesize $\pm$ 0.56} & 82.20 {\footnotesize $\pm$ 0.81} & 71.94 {\footnotesize $\pm$ 2.22} & 91.45 {\footnotesize $\pm$ 0.32} & \textbf{90.26 {\footnotesize $\pm$ 0.48}} \\
& MSLC & 90.94 {\footnotesize $\pm$ 0.45} & 88.36 {\footnotesize $\pm$ 0.80} & 83.93 {\footnotesize $\pm$ 1.21} & 64.90 {\footnotesize $\pm$ 4.84} & \textbf{91.45 {\footnotesize $\pm$ 1.35}} & 89.26 {\footnotesize $\pm$ 0.52} \\
\cmidrule{2-8} 
& \textbf{FasTEN {\scriptsize(ours.)}} & \textbf{91.94 {\footnotesize $\pm$ 0.28}} & \textbf{90.07 {\footnotesize $\pm$ 0.17}} & \textbf{86.78 {\footnotesize $\pm$ 0.31}} & \textbf{79.52 {\footnotesize $\pm$ 0.78}} & \textbf{92.29 {\footnotesize $\pm$ 0.10}} & \textbf{90.43 {\footnotesize $\pm$ 0.31}} \\
\midrule
\multirow{9}{*}{\STAB{\rotatebox[origin=c]{90}{CIFAR-100}}} 
& L2RW & 57.79 {\footnotesize $\pm$ 1.88} & 44.82 {\footnotesize $\pm$ 4.30} & 30.01 {\footnotesize $\pm$ 1.74} & 10.71 {\footnotesize $\pm$ 1.79} & 59.11 {\footnotesize $\pm$ 2.74} & 55.12 {\footnotesize $\pm$ 3.40} \\
& MW-Net & 66.73 {\footnotesize $\pm$ 0.78} & 59.44 {\footnotesize $\pm$ 0.91} & 49.19 {\footnotesize $\pm$ 1.57} & 19.04 {\footnotesize $\pm$ 1.21} & 67.90 {\footnotesize $\pm$ 0.78} & 64.50 {\footnotesize $\pm$ 0.34} \\
& Deep kNN & 59.60 {\footnotesize $\pm$ 0.97} & 52.48 {\footnotesize $\pm$ 1.37} & 39.90 {\footnotesize $\pm$ 0.60} & 23.39 {\footnotesize $\pm$ 0.75} & 57.71 {\footnotesize $\pm$ 0.47} & 50.23 {\footnotesize $\pm$ 1.12}\\
\cmidrule{2-8} 
& GLC & 60.99 {\footnotesize $\pm$ 0.64} & 49.00 {\footnotesize $\pm$ 4.33} & 33.38 {\footnotesize $\pm$ 4.09} & 20.38 {\footnotesize $\pm$ 1.35} & 64.43 {\footnotesize $\pm$ 0.43} & 54.20 {\footnotesize $\pm$ 0.86}\\
& MLoC & \textbf{68.16 {\footnotesize $\pm$ 0.41}} & 62.09 {\footnotesize $\pm$ 0.33} & \textbf{54.49 {\footnotesize $\pm$ 0.92}} & 20.23 {\footnotesize $\pm$ 1.86} & 69.20 {\footnotesize $\pm$ 0.59} & 66.48 {\footnotesize $\pm$ 0.56} \\
\cmidrule{2-8} 
& MLaC & 49.81 {\footnotesize $\pm$ 5.59} & 35.15 {\footnotesize $\pm$ 5.75} & 20.15 {\footnotesize $\pm$ 2.81} & 12.85 {\footnotesize $\pm$ 0.87} & 56.46 {\footnotesize $\pm$ 3.54} & 49.20 {\footnotesize $\pm$ 3.23} \\
& MSLC & \textbf{68.62 {\footnotesize $\pm$ 0.60}} & \textbf{63.30 {\footnotesize $\pm$ 0.49}} & 53.83 {\footnotesize $\pm$ 0.70} & 21.07 {\footnotesize $\pm$ 5.20} & \textbf{70.86 {\footnotesize $\pm$ 0.30}} & \textbf{66.99 {\footnotesize $\pm$ 0.69}} \\
\cmidrule{2-8} 
& \textbf{FasTEN {\scriptsize(ours.)}} & \textbf{68.75 {\footnotesize $\pm$ 0.60}} & \textbf{63.82 {\footnotesize $\pm$ 0.33}} & \textbf{55.22 {\footnotesize $\pm$ 0.64}} & \textbf{37.36 {\footnotesize $\pm$ 1.15}} & \textbf{70.35 {\footnotesize $\pm$ 0.51}} & \textbf{67.93 {\footnotesize $\pm$ 0.53}} \\
\bottomrule
\end{tabular}%
}
\end{table*}

%% file: Tables/2_Evaluation_Results_on_Realworld.tex
\begin{table}[t]
\scriptsize
\centering
\caption{
Test accuracy (\%) comparison on Clothing1M dataset with real-world label noise.
Rows with \(\dagger\)~denote results directly borrowed from~\cite{zheng2021meta} and $^\star$~denotes the result directly borrowed from~\cite{li2021provably}.
All the other results except L2RW~\cite{ren2018learning} are taken from original papers.
}
\label{tab:2_performance_real}
\begin{tabular}{c|lc}
\toprule
& Method & Top-1 Accuracy \\
\midrule
\multirow{7}{*}{\STAB{\rotatebox[origin=c]{90}{Clean set X}}}
& Forward\(^\star\) & 69.91 \\
& T-Revision\(^\star\) & 70.97 \\
& casualNL & 72.24\\
& IF & 72.29 \\
& VolMinNet\(^\star\) & 72.42 \\
& DivideMix & 74.76 \\
& AugDesc & 75.11 \\
\midrule
\multirow{6}{*}{\STAB{\rotatebox[origin=c]{90}{Clean set O}}}
& MLoC & 71.10 \\
& L2RW & 72.04  $\pm$ 0.24 \\
& GLC\(^\dagger\)  & 73.69 \\
& MW-Net\(^\dagger\)  & 73.72 \\
& MSLC  & 74.02 \\
& MLaC\(^\dagger\)  & 75.78 \\
\midrule
\multirow{2}{*}{\STAB{\rotatebox[origin=c]{90}{Ours}}}
& FasTEN w/o LC & 77.07 { $\pm$ 0.52} \\
& \textbf{FasTEN} & \textbf{77.83} \textbf{ $\pm$ 0.17} \\
\bottomrule
\end{tabular}
\end{table}



%% file: Tables/7_Training_Time_Comparison.tex
\begin{table*}[t]
\caption{Training time comparison on CIFAR-10 dataset with 80\% symmetric noise. Time (hours) per total training on a single RTX 2080Ti GPU are provided with the relative ratio compared to our method. }
\label{table:3_appendix_efficiency}
\centering
\small
\resizebox{\textwidth}{!}{%
\begin{tabular}{c|ccccccc|c}
\hline
\toprule
\multirow{2}{*}{Method} & L2RW & MW-NET & Deep kNN & GLC & MLoC & MLaC & MSLC & \textbf{FasTEN} \\
 & \cite{ren2018learning} & \cite{shu2019meta} & \cite{bahri2020deep} & \cite{hendrycks2018using_glc} & \cite{wang2020training} & \cite{zheng2021meta} & \cite{wu2020learning} & \textbf{(ours.)} \\
\midrule
Total Training Time & 4.78 & 2.63 & 2.32 & 2.29 & 7.13 & 10.2 & 2.51 & \multirow{2}{*}{\textbf{1.54}} \\
{\small(Relative to Ours.)} & {\small(3.11x)} & {\small(1.71x)} & {\small(1.51x)} &  {\small(1.49x)} & {\small(4.64x)} & {\small(6.64x)} & {\small(1.64x)} &  \\
\bottomrule
\end{tabular} %
}
\end{table*}

%% file: Tables/3_Evaluation_Results_on_Relabel.tex
\begin{table*}[t]
\caption{
Label correction performance comparison on CIFAR-10 with symmetric 80\% noise. 
Accuracy (\%) and Negative Log Likelihood (NLL) loss are calculated using the true labels before the synthetic noise is injected.
Performance of the trained model on all training samples (Overall) and incorrectly labeled training samples (Incorrect) is measured.
$\dagger$ denotes performance extracted from the meta model.
}
\label{tab:3_performance_relabel}
\centering
\resizebox{\textwidth}{!}{%
\begin{tabular}{cc|ccccccccc|cc}
\toprule
 \multicolumn{2}{c|}{\multirow{2}{*}{Method}} & L2RW & MW-Net & Deep kNN & GLC & MLoC & MLaC & MLaC${}^\dagger$ & MSLC & MSLC${}^\dagger$ & FasTEN & \textbf{FasTEN} \\ 
 &  & \cite{ren2018learning} & \cite{shu2019meta} & \cite{bahri2020deep} & \cite{hendrycks2018using_glc} & \cite{wang2020training} & \multicolumn{2}{c}{\cite{zheng2021meta}} & \multicolumn{2}{c}{\cite{wu2020learning}} & {\small w/o LC} &  {\small\textbf{(ours.)}} \\

\midrule
\multirow{2}{*}{Acc.} &  Overall & 0.4450 & 0.6024 & 0.6471 & 0.6900 & 0.6261 & 0.7567 & 0.7672 & 0.6762 & 0.2821 & 0.7559 & \textbf{0.7847} \\
\cmidrule{2-13}
&  Incorrect & 0.4447 & 0.6024 & 0.6483 & 0.6903 & 0.6257 & 0.7569 & 0.7382 & 0.6755 & 0.2836 & 0.7560 & \textbf{0.7861} \\
\midrule
\multirow{2}{*}{NLL} & Overall & 1.6684 & 1.6961 & 1.6085 & 1.3904 & 1.7492 & 0.9868 & 1.7004 & 1.2694 & 1.5989 & 1.0057 & \textbf{0.8889} \\
\cmidrule{2-13}
 & Incorrect & 1.6674 & 1.6957 & 1.6084 & 1.3881 & 1.7493 & 0.9851 & 1.7299 & 1.2722 & 1.5990 & 1.0033 & \textbf{0.8877} \\
\bottomrule
\end{tabular}%
}
\end{table*}


%% file: Figures/2_3_Robustness_and_Transition_Matrix.tex
\begin{figure*}[t]
\subfloat[Robustness]{\includegraphics[width=0.34\columnwidth]{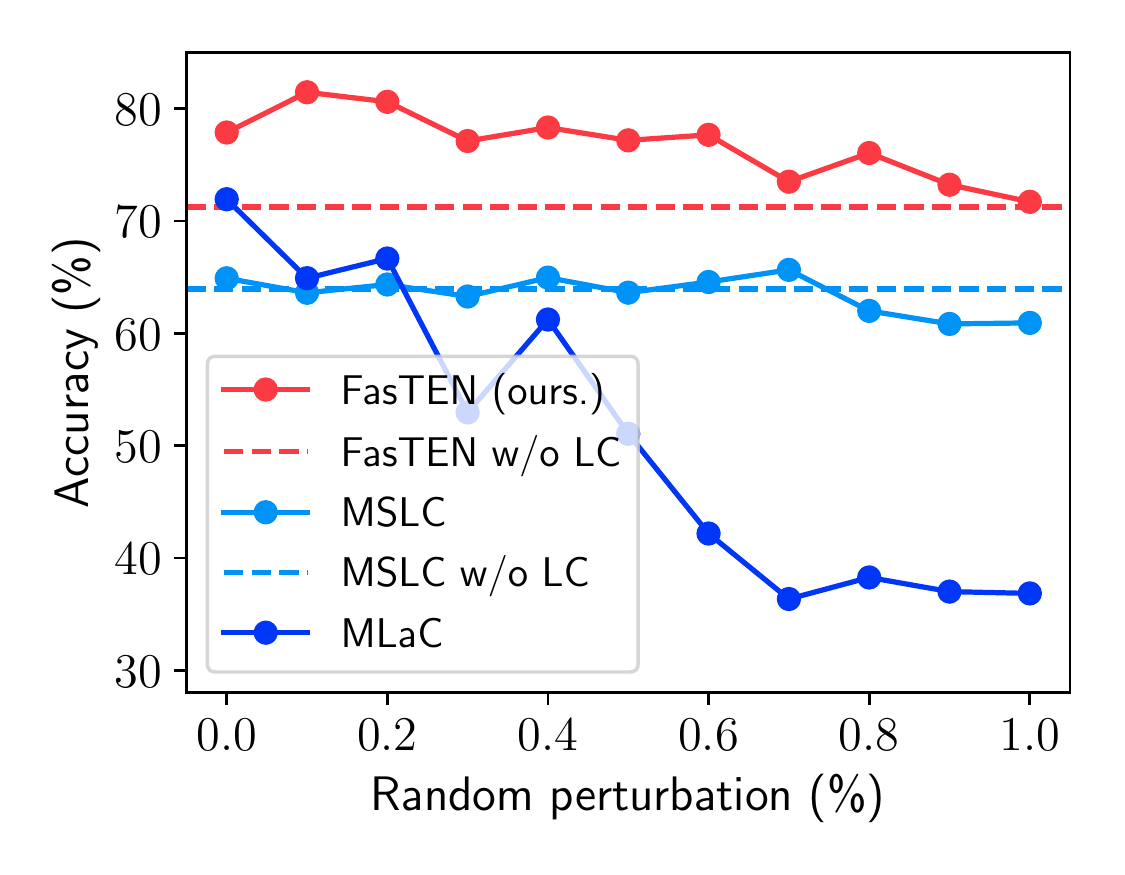}\label{fig:2_robustness}}
\subfloat[Transition matrix estimation]{\includegraphics[width=0.65\columnwidth]{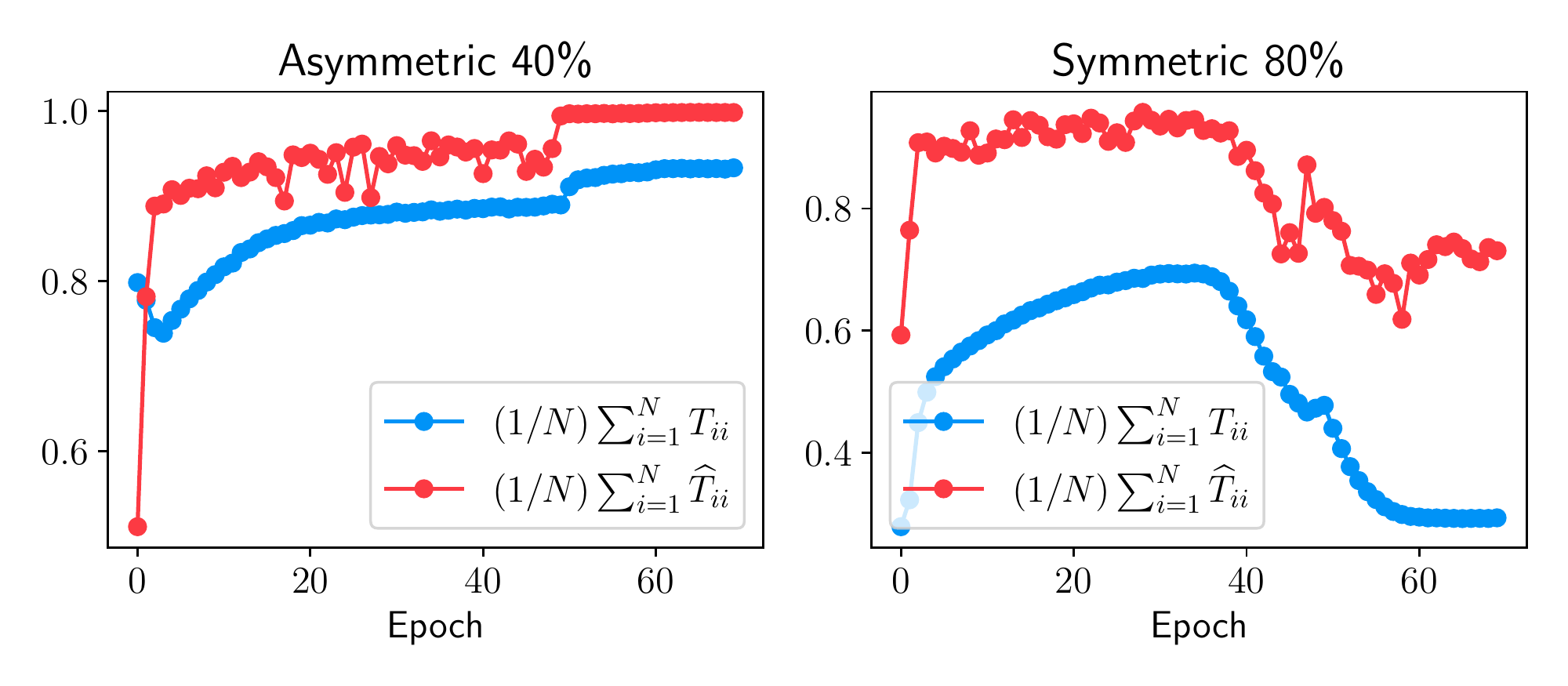}\label{fig:3_transition_matrix}}
\caption{
(a) Robustness to miscorrected labels on CIFAR-10 with various perturbation strength.
Test accuracy (\%) of baselines and baselines without the label correction is provided.
(b) The plot for the mean of the diagonal term in true transition matrix $\mT$ and our estimated transition matrix $\widehat{\mT}$ according to the epoch on CIFAR-10 dataset with  symmetric 80\% and asymmetric 40\% noise.
}
\end{figure*}

%% file: Figures/6_Convergence.tex
\begin{figure}[t]
\centering
\includegraphics[width=0.45\linewidth]{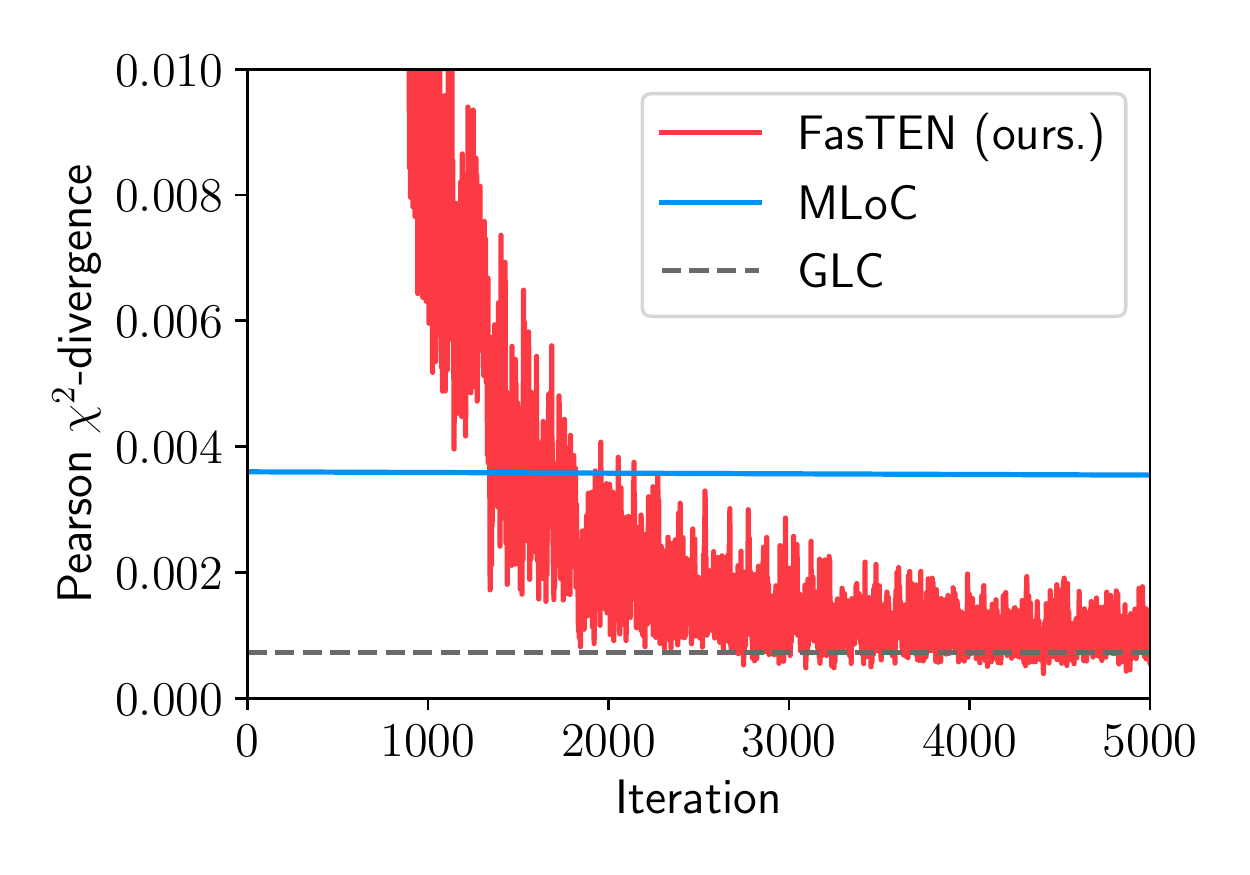}
\caption{
Plot of transition matrix estimation error for every iteration. 
Pearson $\chi^2$-divergence of our FasTEN, MLoC and GLC is provided.
}
\label{fig:convergence}
\end{figure}

%% file: Sections/5_Conclusion.tex
\section{Conclusion}\label{sec:5_conclusion}
In this work, we propose a robust and efficient method, FasTEN, which efficiently learns a label transition matrix that mitigates the label miscorrection problem of existing label correction methods.
Our proposed FasTEN accurately estimates the label transition matrix using a small clean dataset even if the samples are miscorrected.
Moreover, our FasTEN is highly efficient compared to existing methods since it requires single back-propagation through two-head architecture and needs only a single training stage.
Extensive experiments show that our method is the fastest and the most robust classifier.
Especially, our method achieves remarkable performance on both the real-world noise dataset (Clothing1M) and the synthetic dataset on various noise levels (CIFAR).
The detailed analysis shows that our method is robust to miscorrected labels by efficiently estimating the transition matrix shifted by the label correction.


%% file: Sections/A_Theoretical_Analysis.tex
\section{Theoretical Analysis}
\label{appendix:additional_theoretical_analysis}

In this section, we provide a theoretical analysis of our proposed FasTEN.
First, we formally establish the need for label correction (Appendix~\ref{appendix:theory_motivation}).
Then, we provide a formal background for a statistically consistent classifier (Appendix~\ref{appendix:theory_background}) and show the detailed calculation on the estimation of transition matrix $\mT$ through forward propagation (Appendix~\ref{appendix:theory_t}).
Finally, we prove that $\mT$ estimation error of our method are upper bounded (Appendix~\ref{appendix:theory_proof}).

\subsection{Motivation: Need for Label Correction}
\label{appendix:theory_motivation}

\paragraph{Reducing the noise level of a dataset is crucial}
in learning with noisy labels both empirically~\cite{drory2018neural,zheng2021meta,wu2020learning,hendrycks2018using_glc} and theoretically~\cite{chen2019understanding,charikar2017learning}. 
Following \cite{chen2019understanding,ren2018learning,jiang2018mentornet,ma2018dimensionality,han2018co}, 
the true transition matrix $\mT$ in \textit{symmetric noise} with noise level $\gamma \,(< 1)$ is defined as follows:
\begin{equation}
     T_{ij} = \begin{cases}
            1 - \left(\sfrac{(N-1)}{N}\right)\gamma,& \text{for }i = j.\\
            \sfrac{\gamma}{N} , & \text{otherwise.} \\
        \end{cases}   
\end{equation}
Correspondingly, the true transition matrix $\mT$ in \textit{asymmetric noise} with noise level $\gamma \,(< \sfrac{1}{2})$ is defined as follows:
\begin{equation}
     T_{ij} = \begin{cases}
            1 - \gamma,& \text{for }i = j.\\
            \gamma, & \text{for some } i \neq j. \\
            0, & \text{otherwise.}
        \end{cases}   
\end{equation}
We employ these noise schemes in our CIFAR-10 experiments.
Under the assumption that the label class is balanced,
\cite{chen2019understanding} prove that the upper bound of test accuracy for the symmetric and asymmetric noise is as follows: 
\begin{equation}\label{eq:noise_test_acc_upper_bound}
    \begin{cases}
            \left( \sfrac{(N-1)}{N} \right) \gamma^2 - 2\left(\sfrac{(N-1)}{N} \right) \gamma + 1,& \text{for } \textit{symmetric noise}.\\
            2\gamma^2 - 2\gamma + 1, & \text{for } \textit{asymmetric noise}.
        \end{cases}   
\end{equation}

Eq.~\ref{eq:noise_test_acc_upper_bound} shows quadratic (convex) functions of $\gamma$.
In the case of the symmetric noise, test accuracy is minimized at the $\gamma = 1$.
Similarly, with asymmetric noise, test accuracy is minimized at the $\gamma = \sfrac{1}{2}$.
Without additional assumptions, asymmetric noise of $\gamma > \sfrac{1}{2}$ cannot be learned as over half of the training data have wrong labels \cite{han2018co}.
Hence, the test accuracy always decreases as $\gamma$ increases in the feasible bound of $\gamma$.
Therefore, reducing the noise level of a dataset plays a vital role in increasing the achievable test accuracy.

\paragraph{How to Reduce the Noise Level of a Dataset}

Two approaches are commonly used to reduce the noise level of a dataset: re-weighting samples and label correction.
Sample re-weighting reduces the noise level by eliminating noisy samples during the model training, whereas label correction directly cleans up the dataset.
Recently, label correction methods have shown notable results compared to sample re-weighting methods.
\cite{song2019selfie,wu2020learning,zheng2021meta,mirzasoleiman2020coresets,chang2017active,lin2017focal,shrivastava2016training} claim that re-weighting might show sub-optimal performance by filtering out noisy samples, which might aid in training feature extractors.

\paragraph{Theoretical Inspired Explanation of the Superiority of Label Correction}
Here, we provide a more theoretically motivated explanation of the above claim. 
We explore the reason behind the superior performance of label correction compared to sample re-weighting.
\cite{chen2019understanding} considers only the upper bound of test accuracy according to the noise level while ignoring the effect on the number of samples on a generalization error while \cite{charikar2017learning} does not consider deep networks.
We aim to exhibit the superiority of label correction by presenting the generalization error considering both the number of samples and the noise level.
We argue that both the noise level and the number of training samples are critical in determining the generalization error.

For simplicity, our explanation assumes binary classification with asymmetric noise with a level $\gamma$.
We employ the VC dimension framework~\cite{vapnik1999overview,vapnik2013nature,chen2020robustness,zhang2016understanding} to describe the various methods for learning with noisy labels, although the framework provides a loose bound. 
Further investigation on a tighter bound using the Rademacher complexity or considering the multi-class classification is suggested for future research.
Under clean training data distribution $\mathcal{D}$ and clean true data distribution $\mathcal{D}^*$, the VC dimension framework presents the following bound.
\begin{equation}
    p \left( | \mathcal{E}_\mathcal{D} ( f ) - \mathcal{E}_{\mathcal{D}^*} ( f )  | > \epsilon \right) \le 4 \left( 2 | \mathcal{D} | \right)^{\mathbf{d}_{VC}} \exp \left( - \frac{1}{8} \epsilon^2 |\mathcal{D}|  \right),
\end{equation}
where $\mathbf{d}_{VC}$ is the VC dimension and $\mathcal{E}_{\mathcal{D}} (f)$ is the expectation of error for function $f$ regarding the data distribution $\mathcal{D}$.  
If the VC dimension $\mathbf{d}_{VC}$ is bounded (or finite), convergence is guaranteed because the upper bound decreases exponentially as the size of the dataset increases.
Now, we observe a noisy dataset $\Bar{\mathcal{D}}$ rather than a clean dataset $\mathcal{D}$. 
With the triangular inequality and the definition of $\gamma$, the following inequalities hold.
\begin{align}
&p \left( | \mathcal{E}_{\Bar{\mathcal{D}}} ( f ) - \mathcal{E}_{\mathcal{D}^*} ( f )  | > \epsilon \right) \\
=\ & p \left( | \mathcal{E}_{\Bar{\mathcal{D}}} ( f ) - \mathcal{E}_{\mathcal{D}} ( f )  + \mathcal{E}_\mathcal{D} ( f ) - \mathcal{E}_{\mathcal{D}^*} ( f )  |  > \epsilon \right) \\
\le\ & p \left( | \mathcal{E}_{\Bar{\mathcal{D}}} ( f ) - \mathcal{E}_{\mathcal{D}} ( f )  | > \epsilon \right) + p \left( | \mathcal{E}_\mathcal{D} ( f ) - \mathcal{E}_{\mathcal{D}^*} ( f )  |  > \epsilon \right) \\
    \le\  &\gamma +  4 \left( 2 | \mathcal{D} | \right)^{\mathbf{d}_{VC}} \exp \left( - \frac{1}{8} \epsilon^2 |\mathcal{D}|  \right)
\end{align}

Even if this theoretical bound is loose, we argue that label correction shows better performance than re-weighting samples.
As the dataset gets noisier, the number of filtered out samples by the re-weighting methods will also increase, resulting in a drastic reduction of the number of training samples.
However, as aforementioned in Section~\ref{sec:1_introduction}, label correction holds an inherent problem of error propagation. We now explain the theoretical background on how the transition matrix acts as a safeguard for the label correction on our FasTEN.

\subsection{Background: Statistically Consistent Classification}
\label{appendix:theory_background}

It is well known that the label transition matrix $T$ can be used to train \textit{statistically consistent classifiers} in the presence of noisy labels~\cite{mirzasoleiman2020coresets,xia2019anchor,yao2020dual}.
A statistically consistent classifier is a classifier which guarantees the convergence to an optimal classifier when the number of data samples increases indefinitely.
Following \cite{han2020sigua,xia2019anchor,yao2020dual,mirzasoleiman2020coresets}, we describe the consistency of empirical risk to yield the consistency of the classifier.

\paragraph{Statistically Consistent Empirical Risk}
Multi-class classification aims to train the hypothesis $\mathcal{H}$, which estimates a label $y$ given an input $x$.
Given the deep neural network $f_{\phi,\theta}$, a hypothesis $\mathcal{H}$ is commonly defined as follows:
\begin{equation}
    \mathcal{H}(x) = \arg \max_{n \in \{1 , ... ,N \} }  f_{\phi,\theta} (x) |_n .
\end{equation}
With the true sample distribution $\mathcal{D}^*$, the \textit{expected risk} $\mathcal{R}$ for $\mathcal{H}$ is defined as follows:
\begin{equation}
    \mathcal{R}(\mathcal{H}) = \min_{\mathcal{H}} \mathbb{E}_{(x, y) \sim \mathcal{D}^* } \left[ \mathcal{L}(\mathcal{H}(x), y) \right]
\end{equation}
Under the distribution $\mathcal{D}^*$ is unknown, the optimal hypothesis $\mathcal{H}$ should minimize $\mathcal{R}$. 
Since the risk of the optimal hypothesis is difficult to calculate, the empirical risk is usually used for approximation via training dataset $\mathcal{D}$. 
The definition of empirical risk is as follows:
\begin{equation}
\begin{split}
    \mathcal{R}_{|\mathcal{D}|}(\mathcal{H}) & = \mathbb{E}_{(x, y) \sim \mathcal{D} } \left[ \mathcal{L}(\mathcal{H}(x), y) \right]\\ & = \frac{1}{|\mathcal{D}|} \sum_{(x,y) \in \mathcal{D}} \mathcal{L}(\mathcal{H}(x), y) .
\end{split}
\end{equation}
Following equation holds for the \textit{statistically consistent empirical risk}:
\begin{equation}
    \mathcal{R}(\mathcal{H}) = \lim_{|\mathcal{D}| \rightarrow \infty} \mathcal{R}_{|\mathcal{D}|}(\mathcal{H}),
\end{equation}
where it is common to assume that $\mathcal{D}$ is sampled from $\mathcal{D}^*$ as independent and identically distributed (i.i.d) random variables~\cite{xia2019anchor,chen2019understanding,chen2020robustness,han2020sigua,cheng2020learning}.

\paragraph{Statistically Consistent Classifier}

Suppose an ideal zero-one loss function $\mathcal{L}^*$ (where it cannot be used in reality because its differentiation is impossible)~\cite{bartlett2006convexity}:
\begin{equation}
    \mathcal{L}^*(\mathcal{H}(x) = y) = \mathbf{1}_{\{\mathcal{H}(x) \neq y\}}.
\end{equation}
$\mathbf{1}_{\{\cdot\}}$ is an indicator function that outputs 1 if $\mathcal{H}(x) \neq y$ and 0 otherwise.
If the class of the hypothesis $\mathcal{H}$ is large enough~\cite{mohri2018foundations}, the optimal hypothesis to minimize the expected risk $\mathcal{R} (\mathcal{H})$ corresponds to the Bayes classifier~\cite{bartlett2006convexity} as follows:
\begin{equation}
    \mathcal{H}(x) = \arg \max_{n \in \{1, ..., N\}} p (y=n|x)
\end{equation}
Many classification loss functions in modern machine learning are proven to be \textit{classification-calibrated}~\cite{bartlett2006convexity,scott2012calibrated}, i.e., the classification-calibrated loss function leads to a similar prediction to that of $\mathcal{L}^*$ when $|\mathcal{D}|$ is sufficiently large~\cite{mohri2018foundations,vapnik2013nature}.
For example, the hinge loss is proven to be classification calibrated~\cite{yang2020consistency}, and the cross-entropy loss with softmax function is empirically classification-calibrated~\cite{guo2017calibration}.
The classifier $f_{\phi,\theta} (x)$ is said to be statistically consistent when the classifier converges to the probability $p(y|x)$ by minimizing the empirical risk $\mathcal{R}_{|\mathcal{D}|}(\mathcal{H})$.
Note that being risk consistent makes classifier consistent, but not vice versa~\cite{xia2019anchor}.

\paragraph{Statistically Consistent Classifier in Noisy Labels}

The empirical risk $\mathcal{R}_{|\Bar{\mathcal{D}}|}(\mathcal{H})$ of a noisy dataset $\Bar{\mathcal{D}}$ is as follows.
\begin{equation}
\begin{split}
    \mathcal{R}_{\Bar{|\mathcal{D}|}}(\mathcal{H}) 
    & = \mathbb{E}_{(x, y) \sim \Bar{\mathcal{D}} } \left[ \mathcal{L}(\mathcal{H}(x), y) \right] \\& = \frac{1}{|\Bar{\mathcal{D}}|} \sum_{(x,y) \in \Bar{\mathcal{D}}} \mathcal{L}(\mathcal{H}(x), y)
\end{split}
\end{equation}
Since the statistically consistent classifier $f_{\phi, \theta}(x)$ converges to $p(y|x)$, we can accept $f_{\phi, \theta}(x)$ to approximate $p(y|x)$.
Given the definition of the transition matrix $p(\Bar{y}|x) = T^\top p(y|x)$, a hypothesis with a noisy dataset $\Bar{\mathcal{H}}$ is defined as follows:
\begin{equation}
    \Bar{\mathcal{H}}(x) = \arg \max_{n \in \{1 , ... ,N \} }  T^\top f_{\phi,\theta} (x) |_n
\end{equation}
Hence, minimizing the following empirical risk $\mathcal{R}_{\Bar{|\mathcal{D}|}}(\Bar{\mathcal{H}})$  using only the noisy dataset $\Bar{\mathcal{D}}$ leads to a consistent classifier $f_{\phi,\theta} (x)$~\cite{xia2019anchor}.
\begin{equation}
\begin{split}
    \mathcal{R}_{\Bar{|\mathcal{D}|}}(\Bar{\mathcal{H}}) & = \mathbb{E}_{(x, y) \sim \Bar{\mathcal{D}} } \left[ \mathcal{L}(\Bar{\mathcal{H}}(x), y) \right]\\ & = \frac{1}{|\Bar{\mathcal{D}}|} \sum_{(x,y) \in \Bar{\mathcal{D}}} \mathcal{L}(\Bar{\mathcal{H}}(x), y)
\end{split}
\end{equation}
In other words, $f_{\phi,\theta}$ converges to the optimal classifier for the clean data when the sample size of the noisy dataset becomes infinitely large. 
Although other lines of research guarantee that maximizing accuracy in noisy data distribution maximizes accuracy in clean data distribution even without the transition matrix~\cite{chen2020robustness}, loss correction via the transition matrix is still an effective consistent classifier training scheme.
For this reason, a line of work in learning with noisy labels via the transition matrix attempts to train a statistically consistent classifier by an additional layer modeling the transition matrix preceded by the softmax layer~\cite{goldberger2016training,patrini2017making,thekumparampil2018robustness,yu2018learning,mnih2012learning,reed2014training,sukhbaatar2014training}. 
Incidentally, it is known that modifying the loss function using the transition matrix has a degree of handling instance-dependent label corruption~\cite{menon2016learning,hendrycks2018using_glc}.

\paragraph{Statistically Consistent Classifier in Noisy Labels with Small Clean Dataset}
We exploit a small number of clean data as in~\cite{finn2017model,veit2017learning,lee2018cleannet,jiang2018mentornet,ren2018learning,li2019learning,hendrycks2018using_glc,shu2019meta,bahri2020deep,zhang2020self,zheng2021meta,wu2020learning,wang2020training} while disjointing the clean $\mathcal{D}$ and noisy dataset $\Bar{\mathcal{D}}$.
It is trivial that a statistically consistent classifier in exploiting a clean set can be obtained by minimizing the following empirical risk:
\begin{equation}
\begin{split}
    \mathcal{R}_{|\mathcal{D}|}&(\mathcal{H}) +  \mathcal{R}_{\Bar{|\mathcal{D}|}}(\Bar{\mathcal{H}}) \\
    &= \mathbb{E}_{(x, y) \sim \mathcal{D} } \left[ \mathcal{L}(\mathcal{H}(x), y) \right]
    + \mathbb{E}_{(x, y) \sim \Bar{\mathcal{D}} } \left[ \mathcal{L}(\Bar{\mathcal{H}}(x), y) \right] \\
    &= \frac{1}{|\mathcal{D}|} \sum_{(x,y) \in \mathcal{D}} \mathcal{L}(\mathcal{H}(x), y) + \frac{1}{|\Bar{\mathcal{D}}|} \sum_{(x,y) \in \Bar{\mathcal{D}}} \mathcal{L}(\Bar{\mathcal{H}}(x), y).
\end{split}
\end{equation}

Since the cross-entropy loss surrogates the ideal zero-one loss function $\mathcal{L}^*$~\cite{guo2017calibration}, minimizing the empirical risk $\mathcal{R}_{|\mathcal{D}| + |\Bar{\mathcal{D}}|}(\mathcal{H}, \Bar{\mathcal{H}})$ is equivalent to following optimization problem.
\begin{equation}
    \arg \min_{\phi, \theta} \sum_{(x,y)\in \mathcal{D}} \mathcal{L} \left(f_{\phi,\theta} (x), y \right) + \sum_{(x, \Bar{y}) \in \Bar{\mathcal{D}}} \mathcal{L} \left(   \widehat{\mT}^\top f_{\phi,\theta} (x), \Bar{y} \right).
\end{equation}
Without loss of generalization, the optimization problem can be rewritten by introducing an episodic batch formation in Section~\ref{subsec:3_1_episodic_batch_formation}:
\begin{equation}
    \arg \min_{\phi, \theta} \sum_{(x,y)\in d} \mathcal{L} \left(f_{\phi,\theta} (x), y \right) + \sum_{(x, \Bar{y}) \in \Bar{d}} \mathcal{L} \left(   \widehat{\mT}^\top f_{\phi,\theta} (x), \Bar{y} \right).
\end{equation}

\subsection{Calculation of the Estimated Transition Matrix $\widehat{T}$ of FasTEN}
\label{appendix:theory_t}

GLC~\cite{hendrycks2018using_glc} presents a method to estimate the transition matrix through a small clean dataset similar to our FasTEN.
GLC adopts the slow calculation method via a \texttt{FOR} or \texttt{WHILE} loop since \cite{hendrycks2018using_glc} only requires to obtain the transition matrix once in the entire training process. 
However, our FasTEN needs to estimate the transition matrix for every iteration as we correct the labels on the fly, ending up altering the ideal transition matrix.
We speed up the estimation with a single forward propagation by using only matrix operations, avoiding the sluggish \texttt{FOR} or \texttt{WHILE} loop.
Here, we show the derivation of Eq.~\ref{eq:estimation_t}. 
Let $d_i = \{ (x, y) \in d | y_i = 1 \} $. Then, 
\begingroup
\allowdisplaybreaks
\begin{align}
    \widehat{T}_{ij} &= p(\Bar{y}=j|y=i) \\
    &= \frac{1}{|d_i|} \sum_{(x, y) \in d_i } p (\Bar{y} = j | y = i, x) \\
    &= \frac{1}{|d_i|} \left.\sum_{(x, y) \in d_i } f_{\phi, \theta} (x) \right\rvert_j \\
    &= \frac{1}{|d_i|} \left( \left.\sum_{(x, y) \in d_i } y f_{\Bar{\phi}, \Bar{\theta}} (x)^\top \right\rvert_{(i, j)} \right)  \\
    &= \left( \left.\sum_{(x, y) \in d_i } y f_{\Bar{\phi}, \Bar{\theta}} (x)^\top \right\rvert_{(i, j)} \right) \frac{1}{|d_i|} \\
    &= \left( \left.\sum_{(x, y) \in d_i } y f_{\Bar{\phi}, \Bar{\theta}} (x)^\top \right\rvert_{(i, j)} \right) \left( \left. \text{diag}{}^{-1} \left( \sum_{(x, y) \in d} y \right) \right\rvert_{(i,j)} \right)
\end{align}
\endgroup

Without loss of generality, $\widehat{\mT}$ can be written as follows:
\begin{equation}
    \widehat{\mT} = \left( \sum_{(x,y)\in d} y f_{\Bar{\phi}, \Bar{\theta}} (x)^\top \right) \text{diag}{}^{-1} \left( \sum_{(x, y) \in d} y \right).
\end{equation}

\subsection{Proof of Theorem~\ref{theorem:transition_matrix_error}}
\label{appendix:theory_proof}

In this section, we prove under strong assumptions~(Theorem~\ref{theorem:transition_matrix_error_strong}) followed by milder assumptions~(Theorem~\ref{theorem:transition_matrix_error}).
Theorem~\ref{theorem:transition_matrix_error_strong} estimates the upper bound of the error on transition matrix $T$, assuming the ideal situation where $p (\Bar{y} | x)$ is perfectly parameterized to $f_{\Bar{\phi}, \Bar{\theta}} (x)$.

\begin{theorem}
\label{theorem:transition_matrix_error_strong}
Assuming $p (\Bar{y} | x) = f_{\Bar{\phi}, \Bar{\theta}} (x)$, for $\epsilon \ge 0$, 
\begin{equation}
    p \left( \left| \widehat{T}_{ij} - T_{ij} \right| > \epsilon \right) \le 2 \exp \left( -2 \epsilon^2 K  \right).
\end{equation}
\end{theorem}

\begin{proof} 
If $p (\Bar{y} | x) = f_{\Bar{\phi}, \Bar{\theta}} (x)$, then $p \left( \left| \mathbb{E} \left[ \widehat{\mT} \right] - \mT \right| > \epsilon \right) = 0$.
With the triangular and Hoeffding inequality, the following holds:
\begingroup
\allowdisplaybreaks
\begin{align}
    & p \left( \left| \widehat{T}_{ij} - T_{ij} \right| > \epsilon \right) \\
    &= p \left( \left| \widehat{T}_{ij} - \mathbb{E} \left[ \widehat{T}_{ij} \right] + \mathbb{E} \left[ \widehat{T}_{ij} \right] - T_{ij} \right| > \epsilon \right)  \\
    &\le p \left( \left| \mathbb{E} \left[ \widehat{T}_{ij} \right] - T_{ij} \right| > \epsilon \right)  + p \left( \left| \widehat{T}_{ij} - \mathbb{E} \left[ \widehat{T}_{ij} \right] \right| > \epsilon \right)  \\ 
    &= p \left( \left| \mathbb{E} \left[ \widehat{T}_{ij} \right] - T_{ij} \right| > \epsilon \right) +  2 \exp \left( -2 \epsilon^2 K  \right) \\
    &= 2 \exp \left( -2 \epsilon^2 K  \right).
\end{align}
\endgroup
\end{proof}

With Theorem~\ref{theorem:transition_matrix_error_strong} alone, we can see that the estimation error of transition matrix $\mT$ decreases exponentially as $K$ (the number of samples per class) increases, as we mentioned in Theorem~\ref{theorem:transition_matrix_error} of Section~\ref{subsec:3_2_transition_matrix_estimation}.

We assume the hypothetical case that $p (\Bar{y} | x)$ could flawlessly model $f_{\Bar{\phi}, \Bar{\theta}} (x)$, but the assumption does not hold in practice.
Several lemmas are established in order to prove Theorem 1 under more relaxed assumptions.
If $p (\Bar{y} | x) \neq f_{\Bar{\phi}, \Bar{\theta}} (x)$, then $p ( | \mathbb{E} [ \widehat{\mT} ] - \mT | > \epsilon ) \neq 0$. We focus on examining the upper bound of $p ( | \mathbb{E} [ \widehat{\mT} ] - \mT | > \epsilon )$ under the relaxed assumption.
The upper bound of $\widehat{\mT}$ (which is equivalent to $f_{\Bar{\phi}, \Bar{\theta}} (x)$) is strictly $1$ since it is a probability. 
By applying McDiarmid's concentration inequality~\cite{boucheron2013concentration}, the following inequality is established:
\begin{equation}
\begin{split}
    & p \left( \left| \mathbb{E} \left[ \widehat{T}_{ij} \right] - T_{ij} \right| > \epsilon \right) \\
    & \le \mathbb{E}_\sigma \left[ \sup_{\mathcal{H}} \frac{1}{|\mathcal{\Bar{D}}|} \sum_{(x,\Bar{y}) \in \Bar{\mathcal{D}}} \sigma_x \mathcal{L}( \mathcal{H} (x) ,  \Bar{y}) \right]
    + \sqrt{\frac{\log ( 1 / \epsilon )}{2|\Bar{\mathcal{D}}|}}
\end{split}
\end{equation}
where $\sigma$ is an i.i.d Rademacher random variable~\cite{montgomery1990distribution} and $\mathcal{H}$ is a hypothesis.

We estimate the upper bound of the estimation error of $\mT$ by assuming $\mathcal{H}$ is constructed using deep neural networks.
A deep neural networks hypothesis $\mathcal{H}'$ is defined as follows.
\begin{equation}
\mathcal{H}'(x) = \Bar{\theta} \mathcal{A}_{H-1} ( \Bar{\phi}_{H-1} \mathcal{A}_{H-2} (...\ \mathcal{A}_1 (\Bar{\phi}_1 x))) \in \mathbb{R}^N
\end{equation}
where $H$ is the depth of deep neural networks and $\mathcal{A}_i$ is the $i$-th activation function. 
When the function class is limited with deep neural networks, the following lemma holds by borrowing the results of \cite{xia2019anchor}.

\begin{lemma}\label{lemma:rademacher}
Suppose $\sigma$ is an i.i.d Rademacher random variable and $\mathcal{L}$ is the cross-entropy loss function which is $L$-Lipschitz continuous with respect to $\mathcal{H}'$,
\begin{equation}
\begin{split}
\mathbb{E}_\sigma \left[ \sup_{\mathcal{H}} \frac{1}{|\mathcal{\Bar{D}}|} \sum_{(x,\Bar{y}) \in \Bar{\mathcal{D}}} \sigma_x \mathcal{L}( \mathcal{H} (x) ,  \Bar{y}) \right] 
     \le NL \mathbb{E}_\sigma \left[ \sup_{\mathcal{H}'} \frac{1}{|\mathcal{\Bar{D}}|} \sum_{(x,\Bar{y}) \in \Bar{\mathcal{D}}} \sigma_x \mathcal{H}' (x) \right]
\end{split}
\end{equation}
where $\mathcal{H}'$ is a hypothesis belonging to the function class of deep neural networks.
\end{lemma}

As opposed to using the VC dimension framework in Section 1, this section uses the Rademacher complexity framework~\cite{bartlett2002rademacher}  to assess the upper bounds of our method.
Hypothesis complexity of deep neural networks via Rademacher complexity is broadly studied in \cite{xia2019anchor,bartlett2017spectrally,golowich2018size,neyshabur2017pac}.
In particular, \cite{golowich2018size} proves the following lemma:

\begin{lemma}\label{lemma:deep_rademacher}
Assume the Frobenius norm of the weight matrices $\Bar{\phi}_1, ..., \Bar{\phi}_{H-1}, \Bar{\theta}$ are at most $\Bar{\Phi}_1, ..., \Bar{\Phi}_{H-1}, \Bar{\Theta}$ for $H$-layer neural networks $f_{\Bar{\phi}, \Bar{\theta}}$.
Let the activation functions be 1-Lipschitz, positive-homogeneous, and applied element-wise (such as the ReLU). 
Let $\sigma$ is an i.i.d Rademacher random variable.
Let $x$ is upper bounded by B, i.e., for any $x \in \mathcal{X}$, $\|x\| \le B$. Then, for $\epsilon \ge 0$
\begin{equation}
\begin{split}
    \mathbb{E}_\sigma \left[ \sup_{\mathcal{H}'} \frac{1}{|\mathcal{\Bar{D}}|} \sum_{(x,\Bar{y}) \in \Bar{\mathcal{D}}} \sigma_x \mathcal{H}' (x) \right]
    \le \frac{B(\sqrt{ 2H \log 2} + 1) \Bar{\Theta} \Pi_{h=1}^{H-1} \Bar{\Phi}_i }{\sqrt{ | \Bar{\mathcal{D} } |  }}.
\end{split}
\end{equation}
\end{lemma}
Now, we can complete the proof of Theorem~\ref{theorem:transition_matrix_error}.

\begin{proof} With the triangular inequality, Hoeffding inequality, Theorem~\ref{theorem:transition_matrix_error_strong}, Lemma~\ref{lemma:rademacher}, and~\ref{lemma:deep_rademacher}, the following holds.
\begingroup
\allowdisplaybreaks
\begin{align}
    &p \left( \left| \widehat{T}_{ij} - T_{ij} \right| > \epsilon \right)  \\
    &= p \left( \left| \widehat{T}_{ij} - \mathbb{E} \left[ \widehat{T}_{ij} \right] + \mathbb{E} \left[ \widehat{T}_{ij} \right] - T_{ij} \right| > \epsilon \right)  \\
    &\le p \left( \left| \mathbb{E} \left[ \widehat{T}_{ij} \right] - T_{ij} \right| > \epsilon \right)  + p \left( \left| \widehat{T}_{ij} - \mathbb{E} \left[ \widehat{T}_{ij} \right] \right| > \epsilon \right)  \\ 
    &= p \left( \left| \mathbb{E} \left[ \widehat{T}_{ij} \right] - T_{ij} \right| > \epsilon \right) +  2 \exp \left( -2 \epsilon^2 K  \right) \\
    &\le \mathbb{E}_\sigma \left[ \sup_{\mathcal{H}} \frac{1}{|\mathcal{\Bar{D}}|} \sum_{(x,\Bar{y}) \in \Bar{\mathcal{D}}} \sigma_x \mathcal{L}( \mathcal{H} (x) ,  \Bar{y}) \right] + \sqrt{\frac{\log ( \sfrac{1}{\epsilon} )}{2|\Bar{\mathcal{D}}|}} +  2 \exp \left( -2 \epsilon^2 K  \right) \\
    &\le NL \mathbb{E}_\sigma \left[ \sup_{\mathcal{H}'} \frac{1}{|\mathcal{\Bar{D}}|} \sum_{(x,\Bar{y}) \in \Bar{\mathcal{D}}} \sigma_x \mathcal{H}' (x) \right] + \sqrt{\frac{\log ( \sfrac{1}{\epsilon} )}{2|\Bar{\mathcal{D}}|}} +  2 \exp \left( -2 \epsilon^2 K  \right) \\
    &\le \frac{NLB(\sqrt{ 2H \log 2} + 1) \Bar{\Theta} \Pi_{h=1}^{H-1} \Bar{\Phi}_i }{\sqrt{ | \Bar{\mathcal{D} } |  }} + \sqrt{\frac{\log ( \sfrac{1}{\epsilon} )}{2|\Bar{\mathcal{D}}|}}  + 2 \exp \left( -2 \epsilon^2 K  \right)
\end{align}
\endgroup
\end{proof}

Theorem~\ref{theorem:transition_matrix_error} and~\ref{theorem:transition_matrix_error_strong} state that the estimation error of transition matrix $\mT$ is reduced with a larger $K$.
However, we experimentally verify that $K=1$ is enough for achieving comparable performance (See Appendix~\ref{subsubsec:4_4_2_how_sensitive_is_to_the_batch_size}). 

\input{Tables/6_Datasets}
We end this section by enumerating the limitations of our theoretical analysis.
(i) Although our method is based on a multi-head architecture, a clean and a noisy classifier are trained simultaneously, where only the training of the noisy classifier is considered in the theoretical analysis.
(ii) We failed to make the upper bound tight for Theorem~\ref{theorem:transition_matrix_error} and \ref{theorem:transition_matrix_error_strong}.
Additional assumptions like \cite{charikar2017learning} may yield tighter bound.
We conjecture that our method works empirically well for $K=1$ since the upper bound is loose.
(iii) The data distribution on a noisy classifier changes in every iteration due to simultaneously corrected labels.
However, we assume that the data distribution is stationary in the proof.
In order to make our theoretical assumption more adequate for our method, it is necessary to examine the situation under changing data distribution.

%% file: Tables/6_Datasets.tex
\begin{table*}[h]
\caption{
Data split composition of dataset used in our experiments.}
\label{table:desc_dataset}
\centering
\begin{tabular}{c|c|c|c|c|c|c}
\toprule
Dataset & Noisy-train & Clean-train & Valid & Test & Image size & $\#$ of classes \\
\midrule
CIFAR-10 & \multirow{2}{*}{44K} & \multirow{2}{*}{1K} & \multirow{2}{*}{5K} & \multirow{2}{*}{10K} & \multirow{2}{*}{32 $\times$ 32} & 10 \\
CIFAR-100 &  & & &  &  & 100 \\
\midrule
Clothing1M & 1M & 47K & 14K & 10K & 224 $\times$ 224 & 14 \\
\bottomrule
\end{tabular}%
\end{table*}

%% file: Sections/B_Experimental_Details.tex
\section{Experimental Details}
\label{appendix:experimental_details}

\subsection{Datasets}
As shown in Table~\ref{table:desc_dataset}, we use bigger noisy dataset (noisy-train) and smaller clean dataset (clean-train) for training. Validation set is used for obtaining the best model. Since CIFAR datasets do not have the validation set, we split 10\% of the entire training set as the validation set. Thus, experimental results may differ from the results of their papers.
For Clothing1M, we use its original data split.

\subsection{Mini-batch Construction}
We sample the mini-batches from both clean and noisy datasets.
For the clean dataset, we construct the mini-batch to have the same number of instances per class for clean samplers, whereas the mini-batch of the noisy set is randomly sampled.
We choose the batch size 100 for both CIFAR-10 and CIFAR-100, a total of 200 images used per iteration.
As the number of classes is 10 and 100, 10 and 1 image(s) are used per class for the clean batch, respectively.
We choose the batch size 42 for each noisy and clean set on Clothing1M dataset with 14 classes so that 84 images are used per iteration, where 3 samples are used per class for the clean batch.

\subsection{Detailed Training Procedure}

\subsubsection{CIFAR-10/100 dataset}
Here, we describe the detailed training procedure of our baselines on CIFAR-10/100 dataset~\cite{krizhevsky2009learning}.
For all baselines except Deep kNN, we use SGD optimizer with an initial learning rate of 1e-1.
For Deep kNN~\cite{bahri2020deep}, we use Adam optimizer~\cite{kingma2014adam} and set the initial learning rate to 1e-3. 
We follow the experimental settings described in each corresponding papers as much as possible to obtain performance fairly.

\textbf{L2RW:} When training the model, we decay the learning rate to 1e-2 and 1e-3 at 40 and 60 epochs, with a total of 80 epochs.
\textbf{MW-Net:} For the total of 60 epochs, we decay the learning rate by a factor of 10 at 40 and 50 epochs, respectively.
\textbf{Deep kNN:} Since it has multiple training stages, we first train two independent models with only the clean dataset $\mathcal{D}$ and sum of the clean and noisy dataset $\mathcal{D} \cup \Bar{\mathcal{D}}$, respectively.
Then, we filter out the suspicious samples from the noisy set to generate the filtered noisy set $\Bar{\mathcal{D}}_{\text{filter}}$ using $k$-nearest neighbors algorithm ($k$-NN) of the logit outputs from one of the two trained models, where we choose the model with the better validation set accuracy.
Finally, we train the model with the sum of the clean and filtered noisy set $\mathcal{D} \cup \Bar{\mathcal{D}}_{\text{filter}}$.
For each phase, we train the model until 100 epochs without learning rate decay.
\textbf{GLC:} We first train the model with only the noisy set $\Bar{\mathcal{D}}$ and obtain the label transition matrix with the trained model.
With the label transition matrix, we train the initialized model with the clean and noisy set $\mathcal{D} \cup \Bar{\mathcal{D}}$ while correcting the loss obtained from the noisy samples.
\textbf{MLoC:} We first train the model with a warm-up of 75 epochs, i.e., directly training with the noisy label dataset without bells and whistles.
We then meta-train the model with the learning rate of 1e-4 for additional 75 epochs.
\textbf{MLaC:} For a total of 120 epochs, we decay the learning rate at 80 and 100 epochs by a factor of 10.
\textbf{MSLC:} Similar to MLoC, we first train the model with the warm-up of 80 epochs, then meta-train the model with the learning rate of 1e-2 and cut it to 1e-3 at 20 epochs, for 40 epochs. \textbf{FasTEN (Ours):} we train the model until 70 epochs and decay the learning rate at 50 and 60 epochs by a factor of 10.

\subsubsection{Clothing1M dataset}
For the Clothing1M dataset~\cite{xiao2015learning}, we borrow the baseline evaluation results from each corresponding paper except for L2RW, where we train the model ourselves as the original paper does not report the results.
For a fair comparison, we use the same backbone network, ImageNet-pretrained ResNet-50.
\textbf{L2RW:} We train the model for 10 epochs using the SGD optimizer with the initial learning rate of 1e-2.
We decay the learning rate after 5 epochs by a factor of 10, where we follow the common training procedure borrowed from \cite{wu2020learning,shu2019meta}.
\textbf{FasTEN (Ours):} Similarly, we use the SGD optimizer with the same initial learning rate, where we decay the learning rate after 1 epochs for total of 2 epochs.

\subsection{Evaluation Details for Section~\ref{subsec:4_3_relabel_performance}}
\label{appendix:exp_relabel}
Considering the situation where we have to purify the existing noisy labels inside the training set automatically, predicting the correct labels of the train samples is crucial.
We compare the accuracy on the noisy train dataset where we compare with the clean label, which is unknown to the model at training time.
We show the accuracy on CIFAR-10 with symmetric noise of 80\%; hence if the model is perfectly overfitted to the noisy set, it will yield 28\% accuracy.

For each method, we use the model with the best validation accuracy, i.e., the best model that each has produced.
For the meta-model of MLaC, we use the output of the label correction network (LCN), where the network is fed with the feature vector and the noisy label for every noisy dataset to yield a corrected soft label.
The feature vector is extracted from the main model, where it is obtained using the features before the fully connected layer.
For the meta-model of MSLC, we use the cached soft label from the last epoch.
MSLC calculates the soft label with the linear combination of the previous cached soft label, the predicted label from the main model, and the given noisy label, where the weights are continuously learned during training.

%% file: Sections/C_Additional_Experiments.tex
\section{Additional Experiments}
\label{appendix:additional_experiments}

\subsection{Experiments on CIFAR-N with Real-world Noise}
\input{Tables/13_CIFARN}
We conducted further experiments on a recently released dataset, CIFAR-N \cite{wei2022learning}, which contains real-world noise from human annotators.
CIFAR-N is constructed by relabeling the existing CIFAR dataset using the Amazon Mechanical Turk, a crowdsourcing platform, to show the instance-dependent noise from the human annotators.
From the original dataset, we extracted 1K clean samples from the original training samples to construct noisy and clean training subsets.
We use the same training settings of CIFAR-10/100 experiments in Section \ref{subsec:4_1_predictive_performance_comparison} that does not use semi-supervised methods or sophisticated augmentation techniques to evaluate the methods fairly.
Table \ref{tab:appendix_cifarn_eval} shows that our proposed FasTEN achieves better performance than the baselines, similar to the results on Clothing1M.
These additional results demonstrate that FasTEN is more robust against real-world noise.

\subsection{Effect of Label Correction}
\input{Figures/2_Label_Correction}
Figure~\ref{fig:2_label_correction} shows the effect of label correction on our method.
When FasTEN does not correct samples at the end of each epoch, it degrades the predictive performance.
Furthermore, as the noise gets severe, performance further degrades where we expect the effect of label correction to be larger~\cite{chen2019understanding}.
From these observations, we verify that label correction also contributes for improving performance.

\subsection{Robustness to Miscorrected Labels}
Figure~\ref{fig:3_relabel_amount} demonstrates the robustness of our method when unreliable samples are also corrected by lowering the threshold of label correction to investigate how safe our method is.
We observe the robustness of our method compared to other label correction methods, MLaC~\cite{zheng2021meta} and MSLC~\cite{wu2020learning}.
We observe how the performance of our model changes when labels are corrected more unreliably as we lower the threshold $\rho$.
The experiments are conducted on CIFAR-10 with the most severe noise level (symmetric 80\%).
As shown in Figure~\ref{fig:3_relabel_amount}, we verify that our method is robust for miscorrected samples even if all samples in the noisy dataset are corrected when the threshold is under $0.5$.
Our method uses the transition matrix to avoid the error propagation problem even if unreliable samples are corrected.

\subsection{Is the Performance Improvement Due to Over-sampling on the Clean Dataset?}
\label{appendix:exp_oversampling}

Unlike many MAML-based methods using the clean dataset as gradient guidance in the meta training step, our proposed method utilizes the dataset directly during the model training.
One may suspect that the performance improvement of our method may come from over-sampling the clean dataset.
Therefore, we compare our proposed model with an over-sampling method~\cite{chawla2002smote,hong2020disentangling}.
To see the effectiveness of our batch formation, we experiment with the standard cross-entropy loss (Eq.~\ref{eq:CE}) instead of our final objective (Eq.~\ref{eq:final_loss}), using the same batch formation (Naïve Oversampling). For a fair comparison, label correction is excluded. 
\begin{equation}\label{eq:CE}
\arg \min_{\phi, \theta}
\sum_{(x,y)\in d} \mathcal{L} \left(f_{\phi,\theta} (x), y \right) + \sum_{(x, \Bar{y}) \in \Bar{d}}  \mathcal{L} \left( f_{\phi,\theta} (x), \Bar{y} \right)     
\end{equation}

Table~\ref{table:appendix_oversampling} shows that our proposed method outperforms the over-sampling method.
This observation indicates that our meta-learning method appropriately leverages the clean dataset to estimate the label corruption matrix.

\input{Tables/8_Overampling}
\input{Tables/9_Comparison_to_GLC}

\subsection{Comparison to Other Methods with the Transition Matrix}

Although the transition matrix is initially introduced as a safeguard to mitigate the risk of label correction in the FasTEN, our FasTEN even shows better performance than other methods employing the transition matrix.
This section illustrates that FasTEN, even without label correction, shows better performance than other methods using transition matrix with the clean dataset:
GLC (Section~\ref{appendix:exp_glc}) and MLoC (Section~\ref{appendix:exp_mloc}).

\subsubsection{Comparison to Gold Loss Correction (GLC)~\cite{hendrycks2018using_glc}}
\label{appendix:exp_glc}
Our proposed method is similar to GLC in estimating the label transition matrix, but it shows better performance than GLC even without label correction (See Table~\ref{table:appendix_glc}).
Additionally, instead of estimating the transition matrix, we directly use the oracle matrix to examine the effectiveness of the multi-head architecture more clearly.
Even using the same oracle matrix for both methods, our FasTEN outperforms GLC.
We conjecture that our multi-head architecture trains the model to extract features better than the two-stage training of GLC, which learns noisy classifier and clean classifier consecutively.

\subsubsection{Comparison to Meta Loss Correction (MLoC)~\cite{wang2020training}}
\label{appendix:exp_mloc}

Since our method does not directly parameterize the label transition matrix $\mT$, stable estimation of $\mT$ and its theoretical analysis are possible (See Theorem~\ref{theorem:transition_matrix_error}).
Table~\ref{table:appendix_mloc} shows that MLoC and our FasTEN without Label Correction (LC) shows comparable performance.
MLoC uses several engineering techniques for stable training: a strong prior and gradient clipping, where it is not mentioned in the paper.
However, our method shows good performance even without label correction, being robust to different hyperparameters, reducing the need for excessive engineering.
We also emphasize that there is a significant gap in performance at a severe noise level.

\input{Tables/10_Comparison_to_MLoC}
\input{Tables/11_Label_Detection_Performance}



\subsection{Incorrect Label Detection Performance Comparison}
\label{appendix:exp_detection}
We consider the case where we have to continuously purify the already-collected dataset with the existence of a human oracle, where the process can be accelerated by correctly detecting the candidates for the wrongly labeled samples.
Hence, we regard the incorrect label detection problem as a binary classification problem where the model output probability of the noisy samples is used as the barometer for the correctness of the label.

\paragraph{Settings.} 
We extract the probability values of the noisy labels per sample inside the noisy training set to be the negative score for the binary classification problem where we label $1$ for the wrongly labeled sample and $0$ otherwise.
We additionally measure the performance of the meta-models.
We use the meta-learned sample weights for MSLC, MW-Net, and L2RW.
For MLaC, we use the probability of the soft label obtained by the meta-model, which is described in detail in Appendix \ref{appendix:exp_relabel}.
Finally, as Deep kNN filters out the doubtful samples while training the final model, we regard the process as weighting each sample by 0 or 1 depending on its doubtfulness.
Note that we evaluate each method using all the training samples, including the correctly labeled, as each method may mistake those samples to be wrongly labeled.

\paragraph{Results.}
\cite{ren2018learning,shu2019meta,bahri2020deep} claim that using meta-learning or pre-training is able to tell whether a sample is mislabeled.
Although our model is not directly aimed at finding noisy samples in the noisy set, Table~\ref{table:appendix_detection} shows that our model achieves comparable or better performance in detecting noisy labels than the baselines.
Although \cite{ren2018learning,bahri2020deep} claim that the performance has improved because the meta-model detects noisy samples through re-weighting, the actual performance of meta-models is generally lower than that of the final classifier.
It implies that \cite{ren2018learning,bahri2020deep} may operate with different dynamics than the original author intended.

\subsection{Analysis on Our Method}\label{subsec:4_4_analysis_on_our_method}
\subsubsection{How Many Clean Samples are Required?}\label{subsubsec:4_4_1_how_many_clean_sample_are_required}
To verify the effect of the clean dataset size, we observe the performance differences while varying the size.
As shown in Table~\ref{tab:4_analysis_clean}, our method consistently shows better performance on different sizes. 
Especially, even when our method only uses 100 clean samples, it outperforms all the baselines which utilize all the clean samples (1,000 samples).
This observation demonstrates that our method could accurately estimate the label transition matrix with a small number of clean samples, which can be applicable to real-world scenarios where it is difficult to obtain a sufficient number of clean samples.

\subsubsection{How Sensitive is to the Batch Size?}
\label{subsubsec:4_4_2_how_sensitive_is_to_the_batch_size}
\input{Tables/4_Analysis_Clean}
\input{Tables/12_Optimal_Hyperparameter}
In Section~\ref{subsec:3_2_transition_matrix_estimation}, we show that the accuracy of estimating the label transition matrix is upper-bounded by the number of samples in the mini-batch.
As previous studies~\cite{hendrycks2018using_glc} mentioned, the quality of the estimated transition matrix affects the performance in learning with noisy labels.
To verify the effect of the number of samples in the mini-batch, we observe the performance changes by varying the number of samples per class in the mini-batch from 1 to 10.
As shown in Table~\ref{tab:5_analysis_batch}, there is little change in performance depending on the number of samples per class, although the performance degradation is predicted by Theorem~\ref{theorem:transition_matrix_error} when the number of samples is small.
From this observation, we believe that our proposed method shows practicality even in situations where the batch size cannot be increased due to the limited computing resources.

\subsubsection{Searching the Optimal Hyperparameter $\lambda$}
\label{appendix:searching_the_optimal_hyperparameter}
We observe performance variance on the CIFAR-10/-100 datasets when we change the hyperparameter $\lambda$ which is a loss balancing factor. 
The results are summarized in Table~\ref{tab:12_optimal_hyperparameter}.
The hyperparameter is searched in \{0.01, 0.05, 0.1, 0.2, 0.5, 1.0\}.

%% file: Tables/13_CIFARN.tex
\begin{table}[h]
    \caption{Test accuracy (\%) comparison on CIFAR-N dataset with real-world label noise.}
    \label{tab:appendix_cifarn_eval}
    \centering
    \small
    \begin{tabular}{c|cc}
    \toprule
        Methods & CIFAR-10N & CIFAR-100N \\
        \midrule
        L2RW [79] & 83.15 & 57.54 \\
        MW-Net [84] & 83.82 & 61.12 \\
        Deep kNN [3] & 82.87 & 52.71 \\
        \midrule
        GLC [40] & 82.71 & 54.72 \\
        MLoC [100] & 84.45 & 61.49  \\
        \midrule
        MLaC [119] & 83.64 & 45.95 \\
        MSLC [103] & 84.35 & 63.51 \\
        \midrule
        FasTEN (Ours) & \textbf{87.01} & \textbf{63.53} \\
    \bottomrule
    \end{tabular}

\end{table}

%% file: Figures/2_Label_Correction.tex
\begin{figure*}[t]
\begin{minipage}{0.49\textwidth}
\centering
\includegraphics[width=\linewidth]{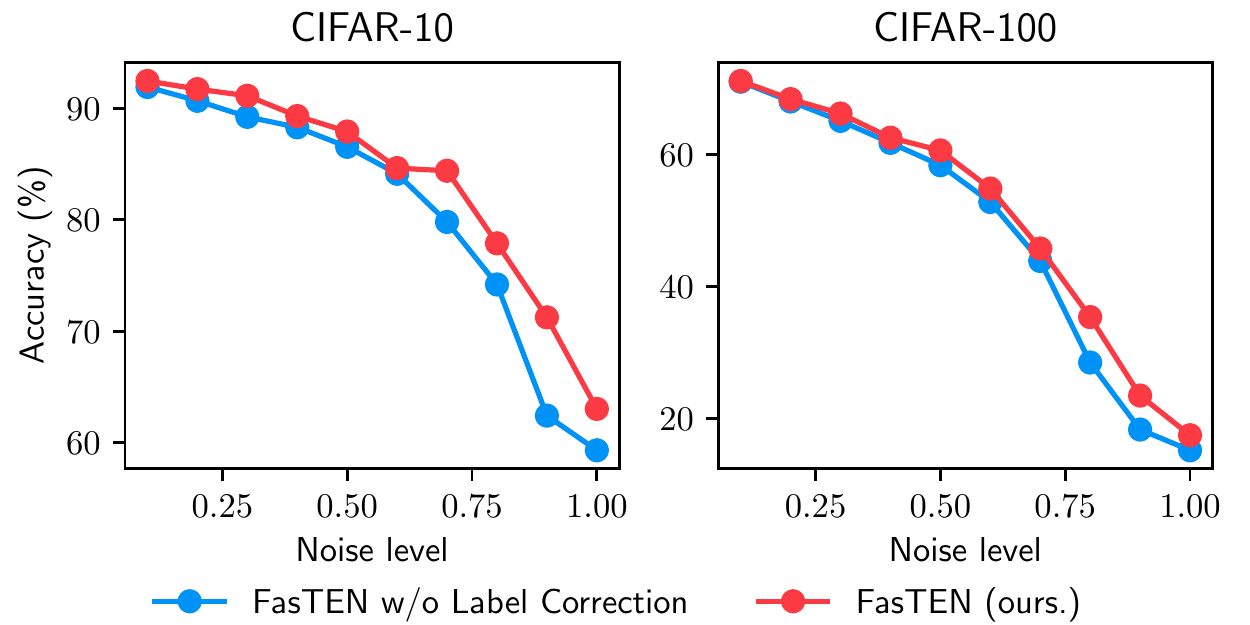}

\caption{
Effect of label correction on CIFAR-10/100 with various symmetric noise ratio. Accuracy (\%) of FasTEN (ours.) and FasTEN without label correction is provided.
}
\label{fig:2_label_correction}
\end{minipage}
\hfill
\begin{minipage}{0.49\textwidth}

\centering
\includegraphics[width=1.0\linewidth]{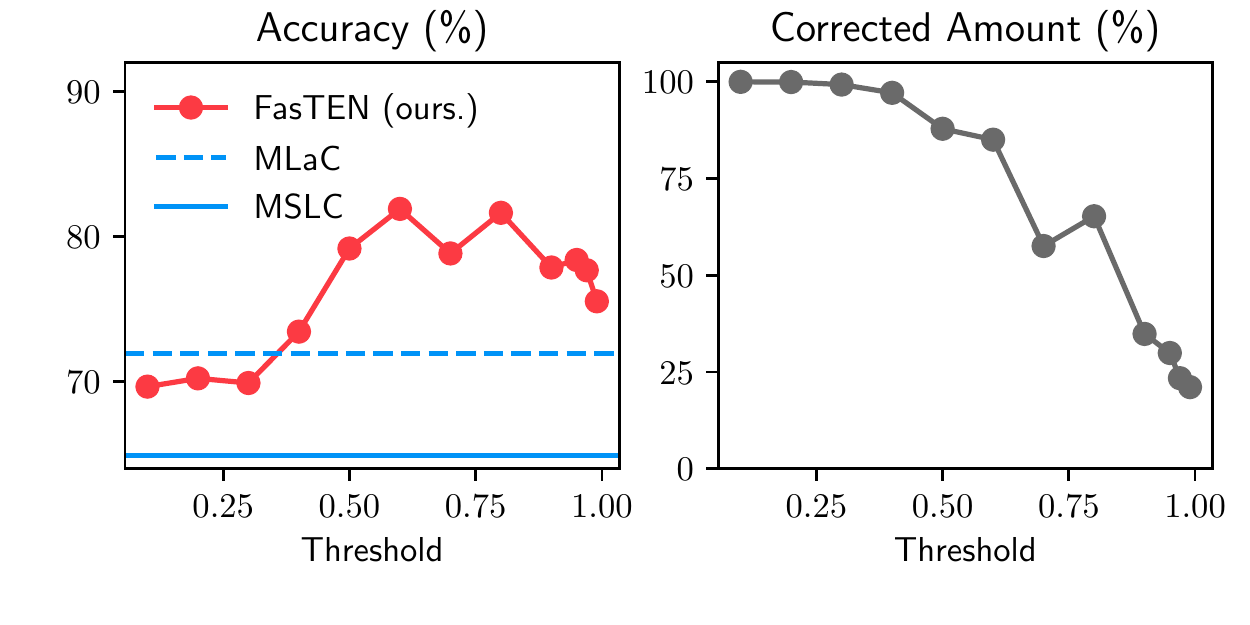}

\caption{
Robustness to Miscorrection of FasTEN. 
The corrected label amount (\%) and model accuracy (\%) according to the label correction threshold ($\rho$) are provided.
}
\label{fig:3_relabel_amount}
\end{minipage}
\end{figure*}

%% file: Tables/8_Overampling.tex
\begin{table*}[h]
\caption{
The effect of clean set oversampling on the performance of CIFAR-10/100 experiments.
The accuracy (\%) of naïve oversampling and FasTEN (ours.) w/o label correction is provided.
}
\label{table:appendix_oversampling}
\centering
\resizebox{1.0\textwidth}{!}{%
\begin{tabular}{c|c|cccc|cc}
\toprule
\multirow{2}{*}{Dataset} & \multirow{2}{*}{Method} & \multicolumn{4}{c|}{Symmetric} & \multicolumn{2}{c}{Asymmetric} \\ 
&  & 20 \% & 40 \% & 60 \% & 80 \% & 20 \% & 40 \% \\
\midrule
\multirow{2}{*}{CIFAR-10} & Naïve Oversampling & 
89.39& 85.90& 83.90& 56.83& 90.58& 84.41
\\
& \textbf{FasTEN (ours.)} w/o Label Correction & \textbf{90.67} & \textbf{88.29} & \textbf{84.12} & \textbf{74.19} & \textbf{92.50} & \textbf{91.05} \\
\hline
\multirow{2}{*}{CIFAR-100} & Naïve Oversampling &
65.42& 57.08& 42.18& 25.88& 67.71& 61.73
\\
& \textbf{FasTEN (ours.)} w/o Label Correction & \textbf{68.02} & \textbf{61.75} & \textbf{52.79} & \textbf{28.46} & \textbf{69.59} & \textbf{66.07} \\
\hline
\bottomrule
\end{tabular}%
}
\end{table*}

%% file: Tables/9_Comparison_to_GLC.tex
\begin{table*}[h]
\caption{
The effect of two-head architecture via oracle label transition matrix on CIFAR-10/100 dataset. 
Test accuracy (\%) of GLC~\cite{hendrycks2018using_glc} and FasTEN (ours.) with and without oracle label transition matrix is provided.
For a fair comparison, label corruption is excluded in FasTEN.}
\label{table:appendix_glc}
\centering
\resizebox{1.0\textwidth}{!}{%
\begin{tabular}{cc|cccc|cc}
\toprule
\multirow{2}{*}{Dataset} & \multirow{2}{*}{Method} & \multicolumn{4}{c|}{Symmetric} & \multicolumn{2}{c}{Asymmetric} \\ 
& & 20 \% & 40 \% & 60 \% & 80 \% & 20 \% & 40 \% \\
\midrule
\multirow{4}{*}{CIFAR-10} & GLC~\cite{hendrycks2018using_glc} w/ Oracle & 
89.06 & 85.45 & 81.56 & 67.54 & 91.74 & 90.35
\\
& \textbf{FasTEN (ours.)} w/ Oracle w/o Label Correction & \textbf{91.37} & \textbf{88.71} & \textbf{83.97} & \textbf{74.91} & \textbf{91.80} & \textbf{91.10} \\
\cmidrule{2-8}
& GLC~\cite{hendrycks2018using_glc} & 
89.66 & 85.30 & 80.34 & 67.44 & 91.56 & 89.76 \\
& \textbf{FasTEN (ours.)} w/o Label Correction & \textbf{90.67} & \textbf{88.29} & \textbf{84.12} & \textbf{74.19} & \textbf{92.50} & \textbf{91.05} \\
\hline
\multirow{2}{*}{CIFAR-100} & GLC~\cite{hendrycks2018using_glc} & 
60.99 &	49.00 &	33.38 &	20.38 &	64.43 &	54.20 \\
& \textbf{FasTEN (ours.)} w/o Label Correction & \textbf{68.02} & \textbf{61.75} & \textbf{52.79} & \textbf{28.46} & \textbf{69.59} & \textbf{66.07} \\
\hline
\bottomrule
\end{tabular}%
}
\end{table*}

%% file: Tables/10_Comparison_to_MLoC.tex
\begin{table*}[h]
\caption{Test accuracy (\%) comparison between FasTEN without Label Correction and MLoC~\cite{wang2020training} on CIFAR-10/100 dataset.}
\label{table:appendix_mloc}
\centering
\resizebox{0.95\textwidth}{!}{%
\begin{tabular}{cc|cccc|cc}
\toprule
\multirow{2}{*}{Dataset} & \multirow{2}{*}{Method} & \multicolumn{4}{c|}{Symmetric} & \multicolumn{2}{c}{Asymmetric} \\ 
& & 20 \% & 40 \% & 60 \% & 80 \% & 20 \% & 40 \% \\
\midrule
\multirow{2}{*}{CIFAR-10} & MLoC~\cite{wang2020training} & 
90.50 & 87.20 & 81.95 & 54.64 & 91.15 & 89.35
\\
& \textbf{FasTEN (ours.)} w/o Label Correction & \textbf{91.37} & \textbf{88.71} & \textbf{83.97} & \textbf{74.91} & \textbf{91.80} & \textbf{91.10} \\
\hline
\multirow{2}{*}{CIFAR-100} & MLoC~\cite{wang2020training} & 
\textbf{68.16} &	\textbf{62.09} &	\textbf{54.49} &	20.23 &	69.20 &	\textbf{66.48} \\
& \textbf{FasTEN (ours.)} w/o Label Correction & 68.02 & 61.75 & 52.79 & \textbf{28.46} & \textbf{69.59} & 66.07 \\
\hline
\bottomrule
\end{tabular}%
}
\end{table*}

%% file: Tables/11_Label_Detection_Performance.tex
\begin{table*}[h]
\footnotesize
\caption{
Incorrect label detection performance comparison on CIFAR-10 with symmetric 80\% noise.
The Area Under the Receiver Operating Characteristic (AUROC) and The Area Under the Precision-Recall Curve (AUPRC) are provided.
Note that pure random model will yield 0.5 AUROC and 0.72 AUPRC.
$\dagger$ denotes the performance of the sample weights obtained with the meta model.
}
\label{table:appendix_detection}
\centering
\resizebox{\textwidth}{!}{%
\begin{tabular}{c|cccccccccc|c}
\toprule
 & L2RW & L2RW${}^\dagger$ & MW-NET & Deep kNN & Deep kNN${}^\dagger$ & GLC & MLoC & MLaC & MLaC${}^\dagger$ & MSLC & \textbf{FasTEN} \\ 
\midrule
AUROC &
0.8653	& 0.4898 & 0.9205	& 0.9019	& 0.8070 & 0.9324	& 0.9318	& 0.9640 & 0.9564 &	0.9303 & \textbf{0.9651} \\
AUPRC &  
0.9412 & 0.7994 & 0.9631 & 0.9396 & 0.9326 & 0.9674 & 0.9695 & \textbf{0.9835} &  0.9791 & 0.9624 & \textbf{0.9835}\\
\hline
\bottomrule
\end{tabular}%
}
\end{table*}


%% file: Tables/4_Analysis_Clean.tex
\begin{table}[t]
\caption{The effect of the number of clean set on CIFAR-10 with symmetric 80\% noise. Comparison with other label correction methods with meta-learning is provided.}
\label{tab:4_analysis_clean}
\centering
\begin{tabular}{c|cc|c}
\toprule
$\#$ of &\multirow{2}{*}{MLaC} & \multirow{2}{*}{MSLC} & \textbf{FasTEN} \\ 
 clean examples &  &  & {\textbf{(ours.)}} \\
\midrule
100  &	32.92	& 69.00	&   \textbf{76.48} \\
250  &	42.15	& 63.52 &	\textbf{79.36} \\
500	 &  50.70	& 63.35 &	\textbf{77.82} \\
1000 &	71.94	& 64.90 &	\textbf{77.88} \\ 
\hline
\bottomrule
\end{tabular}
\end{table}
\begin{table}[t]
\caption{The effect of the number of samples per class ($K$) in the mini-batch on the predictive performance (Accuracy (\%)) of CIFAR-10 experiments.}
\label{tab:5_analysis_batch}
\centering
\begin{tabular}{c|cccc|cc}
\toprule
\multirow{2}{*}{$K$} & \multicolumn{4}{c|}{Symmetric} & \multicolumn{2}{c}{Asymmetric} \\
 & 20 \% & 40 \% & 60 \% & 80 \% & 20 \% & 40 \% \\
\midrule
2	& 91.85	& 89.96	& 87.00	& 81.68	& 92.43	& 91.11  \\
4	& 91.79	& 89.85	& 86.92	& 82.18	& 92.14	& 91.18 \\
6	& 92.16	& 90.07	& 86.77	& 79.57	& 92.14	& 90.98 \\
8	& 92.10	& 89.82	& 84.81	& 79.55	& 92.41	& 90.74 \\
10	& 91.72	& 89.30	& 84.63	& 77.88	& 91.95	& 90.25 \\
\hline
\bottomrule
\end{tabular}%
\end{table}

%% file: Tables/12_Optimal_Hyperparameter.tex
\begin{table*}[t]
\caption{
Evaluation results varying the hyperparameter $\lambda$. Test accuracy (\%) with 95\% confidence interval of 5-runs is provided.}
\label{tab:12_optimal_hyperparameter}
\small
\centering
\begin{tabular}{c|c|cccc|cc}
\toprule
\multirow{2}{*}{} & \multirow{2}{*}{$\lambda$} & \multicolumn{4}{c|}{Symmetric Noise Level} & \multicolumn{2}{c}{Asymmetric Noise Level} \\ 
& & 20\% & 40\% & 60\% & 80\% & 20\% & 40\% \\
\midrule
\multirow{6}{*}{\STAB{\rotatebox[origin=c]{90}{CIFAR-10}}}&0.01 & 90.87 \footnotesize $\pm$ 0.35 & 89.03 {\footnotesize $\pm$ 0.24} & 85.20 {\footnotesize $\pm$ 1.00} & 75.95 {\footnotesize $\pm$ 1.13} & 91.06 {\footnotesize $\pm$ 0.36} & 89.40 {\footnotesize $\pm$ 0.44} \\
&0.05 & 91.69 {\footnotesize $\pm$ 0.17}  & 89.20 {\footnotesize $\pm$ 0.64}  & 84.48 {\footnotesize $\pm$ 0.89}  & 76.39 {\footnotesize $\pm$ 1.79}  & 91.92 {\footnotesize $\pm$ 0.31}  & 89.58 {\footnotesize $\pm$ 0.31}  \\
&0.1 & 91.72 {\footnotesize $\pm$ 0.20} & 89.30 {\footnotesize $\pm$ 0.32} & 84.63 {\footnotesize $\pm$ 0.70} & 77.88 {\footnotesize $\pm$ 1.09} & 91.95 {\footnotesize $\pm$ 0.22}  & 90.25 {\footnotesize $\pm$ 0.39} \\
&0.2 & 91.72 {\footnotesize $\pm$ 0.11}  & 89.61 {\footnotesize $\pm$ 0.29} & 85.71 {\footnotesize $\pm$ 0.24} & 77.55 {\footnotesize $\pm$ 2.78} & 92.20 {\footnotesize $\pm$ 0.19} & 90.51 {\footnotesize $\pm$ 0.27} \\
&0.5 & 91.94 {\footnotesize $\pm$ 0.28} & 90.07 {\footnotesize $\pm$ 0.17} & 86.78 {\footnotesize $\pm$ 0.31} & 79.52 {\footnotesize $\pm$ 0.78} & 92.29 {\footnotesize $\pm$ 0.10} & 90.43 {\footnotesize $\pm$ 0.31} \\
&1.0 & 91.80 {\footnotesize $\pm$ 0.20}  & 89.70 {\footnotesize $\pm$ 0.19} & 86.66 {\footnotesize $\pm$ 0.48} & 80.95 {\footnotesize $\pm$ 0.44} & 92.04 {\footnotesize $\pm$ 0.40} & 90.54 {\footnotesize $\pm$ 0.23} \\\midrule
\multirow{6}{*}{\STAB{\rotatebox[origin=c]{90}{CIFAR-100}}}& 0.01 & 65.36 {\footnotesize $\pm$ 0.77} & 57.79 {\footnotesize $\pm$ 1.12} & 43.65 {\footnotesize $\pm$ 1.06} & 26.95 {\footnotesize $\pm$ 0.76} & 66.58 {\footnotesize $\pm$ 0.71} & 62.19 {\footnotesize $\pm$ 0.66} \\
& 0.05 & 68.49 {\footnotesize $\pm$ 0.27} & 62.47 {\footnotesize $\pm$ 0.32} & 53.55 {\footnotesize $\pm$ 0.86} & 35.53 {\footnotesize $\pm$ 1.28} & 69.73 {\footnotesize $\pm$ 0.18} & 65.63 {\footnotesize $\pm$ 0.79} \\
& 0.1 & 68.38 {\footnotesize $\pm$ 0.29} & 62.53 {\footnotesize $\pm$ 0.33} & 54.82 {\footnotesize $\pm$ 0.46} & 35.35 {\footnotesize $\pm$ 1.13} & 69.35 {\footnotesize $\pm$ 0.13} & 66.34 {\footnotesize $\pm$ 0.27} \\
& 0.2 & 68.65 {\footnotesize $\pm$ 0.09} & 63.07 {\footnotesize $\pm$ 0.22} & 54.84 {\footnotesize $\pm$ 0.30} & 35.65 {\footnotesize $\pm$ 0.66} & 70.37 {\footnotesize $\pm$ 0.15} & 66.93 {\footnotesize $\pm$ 0.20} \\
& 0.5 & 68.75 {\footnotesize $\pm$ 0.60} & 63.82 {\footnotesize $\pm$ 0.33} & 55.22 {\footnotesize $\pm$ 0.64} & 37.36 {\footnotesize $\pm$ 1.15} & 70.35 {\footnotesize $\pm$ 0.51} & 67.93 {\footnotesize $\pm$ 0.53} \\
& 1.0 & 67.91 {\footnotesize $\pm$ 0.59} & 62.78 {\footnotesize $\pm$ 0.28} & 52.76 {\footnotesize $\pm$ 1.15} & 31.45 {\footnotesize $\pm$ 0.75} & 70.02 {\footnotesize $\pm$ 0.60} & 67.11 {\footnotesize $\pm$ 0.55} \\\bottomrule
\end{tabular}%
\end{table*}

%% file: Sections/D_Additional_Related_Work.tex
\section{Additional Related Work}
\label{appendix:additional_related_works}

\subsection{Comparison with Other Methods with Label Transition Matrix}

Under the assumption that label corruption occurs class-dependently and instance-independently, learning with noisy label methods exploiting the label transition matrix has shown admirable performance~\cite{mnih2012learning,reed2014training,sukhbaatar2014training,bekker2016training,patrini2017making,goldberger2016training}.
It is well known that training a statistically consistent classifier is possible if the transition matrix is estimated accurately, but precise estimation is usually challenging~\cite{mirzasoleiman2020coresets,xia2019anchor,yao2020dual}.
Various methods have been proposed to alleviate the issue: imposing strong prior~\cite{patrini2017making,han2018masking}, designing a loss function using the ratio of the label transition matrix $\mT$~\cite{xia2019anchor}, or factorization of the transition matrix~\cite{yao2020dual}.
However, it is still challenging to estimate the transition matrix with only the noisy dataset.
Recently, approaches that improve the estimation accuracy of the transition matrix using a small clean dataset have shown remarkable results: Gold Loss Correction (GLC)~\cite{hendrycks2018using_glc} and Meta Loss Correction (MLoC)~\cite{wang2020training}.
The clean dataset makes it possible to directly estimate the noisy label posterior, resulting in stable prediction of the transition matrix.
Our FasTEN differs from the existing methods which try to find the fixed label transition matrix, as the oracle transition matrix continuously changes during label correction in our method.
Our method shows novelty compared to the previous two methods even without the label correction.

\paragraph{Differences from Gold Loss Correction (GLC)~\cite{hendrycks2018using_glc}}

Similar to FasTEN, GLC models $p (\Bar{y} | x)$ as $f_{\Bar{\phi}, \Bar{\theta}} (x)$ for estimating the transition matrix $\mT$.
However, GLC is more inefficient than our FasTEN because it requires multiple training phases (See $\S$~\ref{subsec:4_2_training_time_comparison}).
We introduce a multi-head architecture with to speed up the training.
Furthermore, there is an additional performance advantage compared to GLC.
The multi-head architecture is presumed to help obtain a better feature extractor by inducing corruption-independent feature extraction.
Detailed experimental results can be found in Appendix~\ref{appendix:exp_glc}.

\paragraph{Differences from Meta Loss Correction (MLoC)~\cite{wang2020training}}

MLoC gradually finds the oracle transition matrix $\mT$ via the MAML framework~\cite{finn2017model}.
As mentioned earlier, MLoC is very slow because it requires three back-propagations for a single iteration due to its nature of MAML (See $\S$~\ref{subsec:4_2_training_time_comparison}).
MLoC directly parameterizes the transition matrix $\mT$ and learns it using various engineering techniques: strong prior and gradient clipping, which were not mentioned in original paper.
In contrast, our method estimates $\mT$ more accurately by sampling the posterior through a single forward propagation.
We empirically validate that our method performs better or comparable to MLoC even without label correction (See Appendix~\ref{appendix:exp_mloc}).

\subsection{Methods using Multi-head Architecture for Noisy Labels}

We propose a multi-head architecture to estimate the transition matrix efficiently: one is for the clean label distribution, and the other is for the noisy label distribution.
A similar multi-head architecture has been used in situations dealing with crowdsourcing.
Many crowdsourcing studies assume that multiple people label a single image~\cite{rodrigues2018deep,guan2018said,tanno2019learning}, where training a reliable classifier is the goal of the crowdsourcing problem.
They maintain separate heads for each annotator, and each head performs multi-task learning to learn each annotator's decisions directly.
Then, the final decision is made by voting each head's decision.
There is no component for estimating the label transition matrix in these methods and no primary head classifier to learn from the estimated label transition matrix.